\documentclass{article}

\usepackage[preprint]{neurips_2026}

\usepackage{amssymb}
\usepackage{amsmath}
\usepackage{amsthm}
\usepackage{graphicx}
\usepackage{subcaption}
\usepackage{mathtools}

\usepackage[utf8]{inputenc} % allow utf-8 input
\usepackage[T1]{fontenc}    % use 8-bit T1 fonts
\usepackage{hyperref}       % hyperlinks
\usepackage{url}            % simple URL typesetting
\usepackage{booktabs}       % professional-quality tables
\usepackage{amsfonts}       % blackboard math symbols
\usepackage{nicefrac}       % compact symbols for 1/2, etc.
\usepackage{microtype}      % microtypography
\usepackage{xcolor}         % colors
\usepackage{tikz}
\usepackage{float}

\newtheorem{theorem}{Theorem}[section]

\newtheorem{definition}{Definition}[section]
\newtheorem{lemma}{Lemma}[section]

\newtheorem{proposition}{Proposition}[section]

\newcommand{\KL}{\operatorname{KL}}

\newcommand{\R}{\mathbb{R}}

\newcommand{\E}{\mathbb{E}}

\newcommand{\norm}[1]{\lVert #1\rVert}
\newcommand{\abs}[1]{\lvert #1\rvert}

\newcommand{\dd}{\,\mathrm d}
\newcommand{\Law}{\operatorname{Law}}

\usepackage{enumitem}
\setlist[itemize]{leftmargin=.5cm}

\usepackage[utf8]{inputenc} % allow utf-8 input
\usepackage[T1]{fontenc}    % use 8-bit T1 fonts
\usepackage{hyperref}       % hyperlinks
\usepackage{url}            % simple URL typesetting
\usepackage{booktabs}       % professional-quality tables
\usepackage{amsfonts}       % blackboard math symbols
\usepackage{nicefrac}       % compact symbols for 1/2, etc.
\usepackage{microtype}      % microtypography
\usepackage{xcolor}         % colors

\title{Dimension-Uniform Discretization Analysis\\ of Preconditioned Annealed Langevin Dynamics\\for Multimodal Gaussian Mixtures} 

\author{%
  Lorenzo Baldassari \\
  University of Basel \\
  \texttt{\texttt{ lorenzo.baldassari@unibas.ch}} \\
  \And
  Josselin Garnier\\
  Ecole Polytechnique, IP Paris\\
  \texttt{ josselin.garnier@polytechnique.edu}\\
  \And
  Knut S\o{}lna\\
  University of California Irvine\\
  \texttt{ksolna@uci.edu}\\
  \And
  Maarten V. de Hoop\\
  Rice University\\
  \texttt{mvd2@rice.edu}\\
}

\begin{document}

\maketitle

\begin{abstract}
Obtaining stable diffusion-based samplers in high- and infinite-dimensional settings is challenging because errors can accumulate across high-frequency coordinates and make the dynamics unstable under refinement of the finite-dimensional approximation of the underlying function-space problem. Discretization is a typical source of such errors, and preconditioning with a suitable spectral decay is one way to control their accumulation. In this paper, we study this problem for preconditioned annealed Langevin dynamics (ALD) applied to Gaussian mixtures. We first show that Euler--Maruyama (EM) discretization, by treating the stiff linear part of the annealed score with a forward Euler step, imposes a stability constraint coupling the preconditioner with the annealed covariance scale. Together with the conditions ensuring dimension-uniform control of the annealed dynamics, this constraint forces the initial smoothed law to remain uniformly close to the target across dimensions. We then consider an exponential-integrator scheme that integrates the stiff linear part of the annealed score exactly. Under explicit spectral summability conditions coupling the smoothing covariance, the component covariance spectra, and the preconditioner, we prove a dimension-uniform Kullback--Leibler (KL) bound  for this scheme. This bound can be made arbitrarily small, uniformly in dimension, by allowing enough time for annealing and then refining the time mesh accordingly. Importantly, these conditions allow regimes in which the KL divergence between the target and the initial smoothed law diverges with dimension, showing that the restrictions imposed by EM are scheme-dependent rather than intrinsic to ALD.
\end{abstract}

\section{Introduction}

Many sampling problems are infinite-dimensional: the target is a probability measure on a  function space, while computation requires a finite-dimensional approximation \cite{stuart2010inverse, cotter2013mcmc, hairer2014spectral}. Designing diffusion-based samplers whose behavior remains stable as this approximation is refined is therefore crucial. The difficulty, and a central challenge in infinite-dimensional sampling,  is that errors that are harmless at a fixed truncation level can accumulate across high-frequency coordinates. Consequently, a sampler may appear stable at any fixed dimension while failing to provide guarantees that are uniform in dimension.
Preconditioning is one way to address this challenge. A preconditioner modifies the geometry of the diffusion; in a diagonal basis, this amounts to rescaling the dynamics across spectral directions. Chosen appropriately, it can improve the behavior of the sampler as the finite-dimensional approximation is refined \cite{roberts2001optimal, cotter2013mcmc}. However, it introduces a trade-off: a less damped preconditioner can accelerate motion in directions where the dynamics would otherwise be slow, but once errors are present, the same preconditioner may amplify their high-frequency components. Thus the choice of preconditioner depends not only on the target and on the continuous-time dynamics, but also on the error sources introduced by the algorithm. Discretization is one such source.

This paper studies this issue for preconditioned annealed Langevin dynamics (ALD), a sampling procedure that applies Langevin dynamics along a path of intermediate  distributions  \cite{guo2024provable}. These distributions are ``annealed'': they interpolate between an initially smoothed version of the target and the target itself, by gradually removing the smoothing \cite{kirkpatrick1983optimization, gelfand1990sampling, neal2001annealed}. The goal is to mitigate the metastability and slow mixing that standard Langevin dynamics can exhibit on multimodal targets, especially in high dimensions \cite{schlichting2019poincare, ma2019sampling, dong2022spectral}: rather than targeting a complex distribution directly, ALD first explores a smoother landscape and is then progressively guided toward the original one.

Despite its empirical success \cite{song2019generative, song2020improved,zilberstein2022annealed}, the theoretical understanding of ALD is still developing \cite{cattiaux2025diffusion}. In finite dimensions, \cite{guo2024provable} has shown that annealing can provide provable advantages over  standard Langevin dynamics for multimodal sampling, including polynomial-in-dimension mixing guarantees. In infinite dimensions, however, far fewer guarantees are available.  One key difficulty is that ALD does not exactly follow the prescribed annealing path: its law may deviate from the intermediate annealed distributions, and this annealing-induced bias must be controlled uniformly under refinement of the finite-dimensional approximations. 
For the continuous-time dynamics, \cite{baldassari2026dimension} has shown that preconditioned ALD targeting Gaussian mixtures---a classical and widely used family for approximating multimodal distributions \cite{mclachlan2000finite, goodfellow2016deep}---can be controlled uniformly in dimension under spectral conditions involving the smoothing covariance, the component covariances, and the preconditioner, and that this control is robust to perturbations of the annealed score. It does not, however, address a question central to implementation: 

\begin{center}
\emph{Under which conditions on the smoothing, component-covariance, and preconditioning spectra does this uniform-in-dimension behavior persist after time discretization?}
\end{center}

In this paper, we answer this question in the same multimodal setting. To isolate the
effect of time discretization, we assume exact access to the 
annealed scores and exact initialization from the initial smoothed law, leaving 
aside the perturbative regime studied in the continuous-time 
analysis of~\cite{baldassari2026dimension}. 

Our first step is to understand which restrictions are introduced by the most standard discretization, Euler--Maruyama (EM) \cite{kloeden2008numerical}. In Section~\ref{sec:em-limitations}, we show that treating the linear part of the annealed score by a forward Euler step imposes a stability condition coupling the preconditioner with the annealed covariance scale. In high-frequency coordinates, where this scale is small, the linear drift can be large (the ``stiff'' directions), so stability requires strong high-frequency damping of the preconditioner. Combined with the assumptions used to keep the annealing bias uniformly controlled, this stability condition further forces the initial smoothed law to remain uniformly close to the target across dimensions. We first prove this for an explicit bimodal family, where the relevant quantities
can be computed directly
(Proposition~\ref{prop:em-true-bias-obstruction}); we then show that
the same implication holds in a more general setting under the
conservative spectral conditions used to control the annealing bias
(Proposition~\ref{prop:em-kl}).

Since the difficulty with EM in infinite dimensions comes from the way it discretizes the stiff linear part of the annealed score, one may expect that alternative schemes treating this part differently could lead to more favorable conditions. In Section~\ref{sec:elp}, we study such an alternative:  an exponential-integrator scheme \cite{shi2012convergence,komori2014stochastic,zhang2022fast} of the type suggested by \cite{guo2024provable} in their finite-dimensional analysis of ALD.
We call it an exact-linear-part (ELP) scheme because, at each time step, it integrates the stiff linear part of the annealed score exactly and freezes the remaining nonlinear term.  Our main result (Theorem~\ref{thm:concrete-elp}) proves that ELP admits a dimension-uniform KL bound under explicit spectral summability conditions on the smoothing, component covariances, and the
preconditioner, which can be less restrictive than those for EM: in particular, they do not force the initial smoothed law to remain uniformly close to the target across dimensions; see
Proposition~\ref{prop:elp-example} and the numerical illustrations of Section~\ref{sec:numerics}. These conditions also clarify how the preconditioner should be chosen: it must be strong enough to keep the annealing bias under control, but sufficiently damped in the tail to prevent discretization errors from accumulating across high-frequency coordinates. Appendix~\ref{app:power-law} makes this trade-off explicit in a
power-law regime, yielding admissible  preconditioning spectra.

\paragraph{Related Work.}
Annealing-based strategies have a long history in sampling  \cite{gelfand1990sampling,marinari1992simulated, neal1996sampling, neal2001annealed}. The basic idea is to replace the direct sampling problem by a sequence of easier intermediate problems, starting from a distribution that is easy to sample from and gradually transforming it into the target. If consecutive distributions are sufficiently close, approximate samples from one stage can provide a warm start for the next, thereby easing the transition toward the target. When this principle is implemented through a Langevin diffusion driven by a time-dependent score schedule, one obtains ALD \cite{song2019generative, song2020improved, zilberstein2022annealed, zilberstein2024solving}.

Recent work has begun to develop theoretical guarantees for ALD, both from the sampling and generative-modeling perspectives
\cite{guo2024provable, chehab2024provable, cattiaux2025diffusion, cordero2025non}.
These analyses, however, do not address whether the behavior of ALD remains stable under successive finite-dimensional approximations of the infinite-dimensional target. 
Available discretization analyses are likewise not uniform in dimension: for instance, \cite{cordero2025non} analyzes an
Euler--Maruyama discretization of ALD with dimension-dependent bounds,
whereas \cite{guo2024provable} is closer in spirit to our approach and uses
an exponential-integrator discretization \cite{shi2012convergence,komori2014stochastic, zhang2022fast}, but still obtains
bounds with polynomial dependence on dimension. 

Recently, \cite{baldassari2026dimension} addressed the infinite-dimensional problem by studying a preconditioned version of ALD. This analysis is closest to ours in spirit: it shows how dimension-uniform control can be obtained from the spectral properties of the annealing geometry and the preconditioner, and identifies conditions under which score and initialization perturbations remain compatible with such robustness. That analysis, however, is limited to continuous time. Our paper shows that the dimension-uniform control proved there is not automatically inherited by time discretization: as the finite-dimensional approximation of the underlying infinite-dimensional target is refined, Euler--Maruyama imposes a restrictive high-frequency condition that the exact-linear-part exponential integrator avoids.

Finally, our analysis connects naturally to the classical function-space Markov chain Monte Carlo literature, which develops sampling schemes designed to yield dimension-independent behavior
\cite{cotter2013mcmc,hairer2014spectral,cui2016dimension,beskos2017geometric}.
It also contributes to a growing line of work emphasizing the role of preconditioning in improving the efficiency and stability of sampling dynamics \cite{bronasco2025efficient, rey2016improving, best2010coordinate, hummer2005position, roberts2002langevin, cui2024optimal, lelievre2024optimizing, lelievre2025improving}.
Given the close connection between ALD and score-based diffusion models, our work further relates to recent efforts to analyze these models in function space
\cite{kerrigan2022diffusion,franzese2023continuous,baldassari2023conditional,pidstrigach2024infinite,bond2024diff,baldassari2024taming,lim2025score,hagemann2025multilevel,franzese2025generative,baldassari2025preconditioned}, although here we take a purely sampling perspective and assume access to the score.

\section{Continuous-Time Background}\label{sec:ct-background}

In this section, we briefly recall the infinite-dimensional framework of \cite{baldassari2026dimension}, which provides the modeling assumptions and truncation regime used throughout the paper. We then introduce the continuous-time preconditioned annealed Langevin dynamics at the core of our analysis.

\paragraph{Infinite-Dimensional Setting and Truncation Regime.}
%\label{sec:setting}

Let $H$ be a separable Hilbert space with orthonormal basis $(e_j)_{j\geq 1}$, and consider the infinite-dimensional Gaussian mixture
\[
\rho_\star^\infty
=
\sum_{i\in I} w_i\,\mathcal N(m_i,\Sigma_i),
\]
where $I$ is finite or countable, $w_i>0$, and $\sum_{i\in I} w_i=1$. We work in a 
 diagonal  setting,
\[
m_i=\sum_{j\geq 1} m_{ij} e_j,
\qquad
\Sigma_i e_j=\sigma_{ij} e_j,
\qquad
\sigma_{ij}>0,
\]
and assume the finite-energy condition
\[
\sum_{j\geq 1} m_{ij}^2<\infty,
\qquad
\sum_{j\geq 1} \sigma_{ij}<\infty,
\qquad
i\in I,
\]
so that $m_i\in H$ and $\Sigma_i$ is trace class. Hence each component $\mathcal N(m_i,\Sigma_i)$ is well defined on $H$ \cite{bogachev1998gaussian}.

In this paper, rather than sampling $\rho_\star^\infty$ directly, we study its successive finite-dimensional truncations. For each $d\geq 1$, let $P_d$ denote the orthogonal projection onto $\mathrm{span}\{e_1,\dots,e_d\}$, and define
\[
\rho_\star^d
:=
(P_d)_\#\rho_\star^\infty
=
\sum_{i\in I} w_i\,\mathcal N(m_i^d,\Sigma_i^d),
\]
where $m_i^d:=(m_{i1},\dots,m_{id})$, $\Sigma_i^d:=\operatorname{Diag}(\sigma_{i1},\dots,\sigma_{id})$.
The diagonal formulation makes the refinement structure explicit: increasing $d$ adds new coordinates in a fixed basis while leaving the previously retained coordinates unchanged. This is the regime in which we study dimension-uniform control. Importantly, this assumption is less restrictive than it may first appear: at each fixed truncation level, diagonal Gaussian mixtures remain highly expressive and, with sufficiently many components, can approximate general mixtures in KL distance. 

For the discretization analysis developed in Sections~\ref{sec:em-limitations} and~\ref{sec:elp}, we will also use the quantities
\[
\underline{\sigma}_j:=\inf_{i\in I}\sigma_{ij},
\qquad
\overline{\sigma}_j:=\sup_{i\in I}\sigma_{ij},
\qquad
\overline m_j:=\sup_{i\in I}|m_{ij}|,
\qquad j\ge1,
\]
and we assume that
$
0<\underline{\sigma}_j\le \overline{\sigma}_j<\infty$ and $\overline m_j<\infty$ for all 
 $j\ge1$.

\paragraph{Annealed Langevin Dynamics.}

We now recall the continuous-time annealed Langevin dynamics (ALD) studied in \cite{baldassari2026dimension}. 
For a time horizon $T>0$, we define for $t\in[0,T]$:
\[
\kappa_t:=\frac{T-t}{T},
\qquad
\rho_t^d
=
\rho_\star^d * \mathcal N(0,\kappa_t C^d),
\quad \mbox{ with }
C^d=\operatorname{Diag}(\lambda_1,\dots,\lambda_d)  ,
\]
and we assume $\sum_{j\ge 1} \lambda_j <\infty$, so that the smoothing covariance is trace class and the corresponding Gaussian perturbation is $H$-valued.
Then $\rho_t^d$ interpolates between the  smoothed law $\rho_0^d=\rho_\star^d * \mathcal N(0,C^d)$ 
and the final target
$\rho_T^d=\rho_\star^d$. 

Starting from $X_0^d\sim \rho_0^d$, consider the preconditioned annealed Langevin diffusion
\begin{equation}\label{eq:ald-ct-background}
dX_t^d
=
\Gamma^d \nabla \log \rho_t^d(X_t^d)\,dt
+
\sqrt{2\Gamma^d}\,dW_t^d,
\qquad
t\in[0,T],
\end{equation}
where $W_t^d$ is a standard Brownian motion in $\mathbb R^d$ and
\[
\Gamma^d=\operatorname{Diag}(\gamma_1,\dots,\gamma_d)
\]
is a diagonal preconditioner, whose role will be discussed later. We denote the terminal law  by
\[
\rho_T^{\mathrm{ALD},d}:=\operatorname{Law}(X_T^d).
\]

Note that \eqref{eq:ald-ct-background} does not target  $\rho_\star^d$ directly: its drift uses the time-dependent score $\nabla\log\rho_t^d$ rather than the final score $\nabla\log\rho_\star^d$. Accordingly, its law   does not in general coincide with the annealing path, and one may have $\rho_T^{\mathrm{ALD},d}\neq \rho_\star^d$.
We quantify this \emph{annealing-induced bias} via the KL divergence
\[
\mathcal B_{\mathrm{ann}}^d(T)
:=
\operatorname{KL}(\rho_\star^d\,\|\,\rho_T^{\mathrm{ALD},d}).
\]

A natural question is whether one can choose a single dimension-uniform $T$ so that $\mathcal B_{\mathrm{ann}}^d(T)$ remains small over successive refinements $(\rho_\star^d)_{d\geq 1}$ of the infinite-dimensional target. An explicit answer was recently given in \cite{baldassari2026dimension}, through the following theorem.
\begin{theorem}
\label{thm:ct-background}
Let
\[
\mathcal K_d
:=
\frac{1}{16}\sum_{i\in I} w_i \sum_{j=1}^d
\frac{\lambda_j}{\gamma_j}
\log\!\Big(1+\frac{\lambda_j}{\sigma_{ij}}\Big).
\]
If $\sup_{d\geq 1}\mathcal K_d<\infty$, then for every $\varepsilon>0$, choosing
\[
T_\varepsilon=\varepsilon^{-1}\sup_{d\geq 1}\mathcal K_d
,
\]
one has
\[
\operatorname{KL}(\rho_\star^d\,\|\,\rho_{T_\varepsilon}^{\mathrm{ALD},d})\leq \varepsilon ,
\qquad\text{for all }d\geq 1.
\]
\end{theorem}

A convenient sufficient condition for $\sup_d \mathcal K_d<\infty$ is obtained from $\log(1+u)\leq u$:
\begin{equation}\label{eq:ct-sufficient}
\sum_{i\in I} w_i \sum_{j\geq 1}
\frac{\lambda_j^2}{\gamma_j\sigma_{ij}}
<\infty.
\end{equation}
Hence, dimension-uniform control reduces to a spectral compatibility condition between the smoothing $(\lambda_j)$, the preconditioner $(\gamma_j)$, and the target covariance spectra $(\sigma_{ij})$.

\paragraph{On the Role of the Preconditioner in Continuous Time.}

At first sight, Theorem~\ref{thm:ct-background} may suggest that preconditioning is a \emph{cost-free accelerator}: the annealing-bias bound depends on $\Gamma^d$ only through the ratios $\lambda_j/\gamma_j$, so it seems that making the preconditioner larger can only help. One might even be tempted to read the theorem as saying that by choosing rapidly growing preconditioner coefficients $(\gamma_j)$ one can make the required continuous-time horizon arbitrarily small, essentially without trade-off.
This interpretation is misleading. Theorem~\ref{thm:ct-background} is idealized: it assumes exact access to the annealed score and exact initialization from $\rho_0^d$, and therefore isolates only the annealing-bias contribution. The robustness analysis in \cite{baldassari2026dimension} shows that, once score perturbations or initialization mismatch are introduced, sufficient decay of $(\gamma_j)$ helps prevent high-frequency errors from accumulating across coordinates and destroying dimension-uniform control. This is the main lesson we import from the continuous-time analysis. Since time discretization also introduces coordinatewise errors, one should expect a similar phenomenon there as well.

The remainder of the paper makes this point explicit. We begin with the most classical discretization scheme, Euler--Maruyama \cite{kloeden2008numerical}. We show that, as the finite-dimensional approximations of the infinite-dimensional target are refined, Euler--Maruyama requires strong high-frequency damping of the preconditioner. This motivates the alternative scheme analyzed in Section~\ref{sec:elp}.

\section{High-Frequency Stiffness in Euler--Maruyama Discretization}\label{sec:em-limitations}

Here, we consider the Euler--Maruyama (EM) discretization of the ALD diffusion. We focus on a specific limitation arising from the fact that EM applies a forward Euler step to the linear part of the annealed score. Along high-frequency coordinates, where the annealed covariance scale is small, this linear part can have large drift coefficients; stability then requires strong high-frequency damping of the preconditioner. Combined with the spectral condition used to control the annealing-induced bias in \eqref{eq:ct-sufficient}, this stability restriction forces a regime in which the smoothed law used to initialize ALD must remain uniformly close to the target across dimensions.

In this section, we continue to work in the diagonal Gaussian-mixture setting of Section~\ref{sec:ct-background}. We write
\[
D_{i,t}^d:=\Sigma_i^d+\kappa_t C^d.
\]
Since the annealed score is
\[
\nabla\log\rho_t^d(x)
=
\sum_{i\in I} p_{i,t}^d(x)\Big(-(D_{i,t}^d)^{-1}(x-m_i^d)\Big),
\mbox{ with responsibilities }
p_{i,t}^d(x)
=
\frac{w_i\varphi_{i,t}^d(x)}{\sum_{k\in I}w_k\varphi_{k,t}^d(x)},
\]
and $\varphi_{i,t}^d$ is the probability density function of the distribution $\mathcal N(m_i^d,D_{i,t}^d)$,
it is useful to rewrite it as
\begin{equation}\label{eq:score-decomp-em}
\nabla\log\rho_t^d(x)
=
-(B_t^d)^{-1}x+G_t^d(x),
\end{equation}
where $B_t^d:=\operatorname{Diag}(b_{t,1},\dots,b_{t,d})$, $
b_{t,j}:=\underline{\sigma}_j+\kappa_t\lambda_j$, and
\[
G_t^d(x)
:=
\sum_{i\in I} p_{i,t}^d(x)
\Big(
\big((B_t^d)^{-1}-(D_{i,t}^d)^{-1}\big)x
+
(D_{i,t}^d)^{-1}m_i^d
\Big).
\]
This decomposition isolates a diagonal linear part, $-(B_t^d)^{-1}x$, whose high-frequency coefficients can become large, from the nonlinear mixture correction $G_t^d(x)$.

With a constant step size $h>0$ and grid points $t_n=nh$, the EM scheme reads
\begin{equation}\label{eq:em-scheme}
Y_{n+1}^d
=
Y_n^d
-
h\,\Gamma^d(B_{t_n}^d)^{-1}Y_n^d
+
h\,\Gamma^d G_{t_n}^d(Y_n^d)
+
\sqrt{2h\,\Gamma^d}\,\xi_n,
\end{equation}
where $\xi_n\sim \mathcal N(0,I_d)$ are i.i.d. Here, the linear and nonlinear terms are treated in the same way: both are frozen at time $t_n$ and advanced by a forward Euler step. For the linear part, this yields coordinate-wise the factor 
$
1-h {\gamma_j}/{b_{t_n,j}} 
$.
Stability requires this factor to have modulus at most one, giving the following immediate condition.

\begin{proposition}\label{prop:em-linear-stability}
The coordinate-wise linear factors in \eqref{eq:em-scheme} satisfy
$\big|1-h {\gamma_j}/{b_{t_n,j}}\big|\leq 1$
for all mesh indices and all $j\geq 1$
if and only if
\[
h\,\sup_{n,j}\frac{\gamma_j}{\underline{\sigma}_j+\kappa_{t_n}\lambda_j}\leq 2.
\]
\end{proposition}
Proposition~\ref{prop:em-linear-stability} highlights two possible regimes.
If $\lambda_j \gtrsim \underline{\sigma}_j$ along the tail, then
$b_{t_0,j}=\underline{\sigma}_j+\lambda_j\asymp\lambda_j$, so the condition above yields
$
\sup_{j\ge1} {\gamma_j}/{\lambda_j}<\infty
$.
Together with the annealing-bias condition \eqref{eq:ct-sufficient}, this forces
\[
\sum_{i\in I} w_i\sum_{j\ge1}\frac{\lambda_j}{\sigma_{ij}}<\infty.
\]
This case is very restrictive. For example, it fails whenever there exists a component $i_0$ with 
$\sigma_{i_0 j}\asymp \underline{\sigma}_j$ along the tail,
and in particular whenever the mixture is finite. Therefore, the regime that is more relevant for our analysis is the more favorable one
\[
\lambda_j \lesssim \underline{\sigma}_j.
\]
Under this condition, Proposition~\ref{prop:em-linear-stability} implies that uniform control of the EM linear factors with a fixed
step size requires
\begin{equation}\label{eq:em-stability-condition}
\sup_{j\ge1}\frac{\gamma_j}{\underline{\sigma}_j}<\infty.
\end{equation}
In the successive-refinement regimes of interest, where $\underline{\sigma}_j\to 0$ as $j\to\infty$, this EM stability condition rules out any non-decaying, let alone increasing, preconditioner. The next proposition shows what this restriction entails. Its proof is given in Appendix~\ref{app:em-true-bias-obstruction}. In a simple symmetric bimodal family, uniform-in-dimension control of the annealing-induced bias together with \eqref{eq:em-stability-condition} forces the smoothed law used to initialize ALD to remain uniformly close to the target across dimensions.
\begin{proposition}
\label{prop:em-true-bias-obstruction}
Let $e_1^d$ be the first canonical basis vector of $\mathbb R^d$. Consider the
symmetric bimodal target
$
\pi_\star^d
=
\frac12\mathcal N(ae_1^d,\Sigma^d)
+
\frac12\mathcal N(-ae_1^d,\Sigma^d)
$,
where $a\neq0$ is fixed and
$\Sigma^d=\operatorname{Diag}(\sigma_1,\dots,\sigma_d)$. Let
$
\pi_t^d=\pi_\star^d*\mathcal N(0,\kappa_t C^d)
$,
with
$
C^d=\operatorname{Diag}(\lambda_1,\dots,\lambda_d)
$
and
$
\Gamma^d=\operatorname{Diag}(\gamma_1,\dots,\gamma_d)
$.
Let $(X_t^d)_{t\in[0,T]}$ be the corresponding ALD diffusion initialized from
$\pi_0^d  =\pi_\star^d * \mathcal N(0,C^d)$.

Assume the EM stability condition
$
\sup_{j\ge2}\gamma_j/\sigma_j<\infty
$.
If, for some fixed $T>0$,
\[
\sup_{d\ge1}\KL(\pi_\star^d\,\|\,\Law(X_T^d))<\infty,
\]
then
\[
\sup_{d\ge1}\KL(\pi_\star^d\,\|\,\pi_0^d)<\infty.
\]
\end{proposition}

If one now returns to the general diagonal Gaussian-mixture setting and combines the EM stability condition with the annealing-bias condition \eqref{eq:ct-sufficient}, the same conclusion holds. Indeed, assume the EM stability condition
$\sup_{j\ge 1} {\gamma_j}/{\underline{\sigma}_j}<\infty$, 
then there exists $M<\infty$ such that, for every $i$ and $j$,
$\gamma_j/\sigma_{ij}\le \gamma_j/\underline{\sigma}_j\le M$, and hence
$\lambda_j^2/\sigma_{ij}^2\le M\,\lambda_j^2/(\gamma_j\sigma_{ij})$. Summing in $i$ and $j$ and using~\eqref{eq:ct-sufficient} gives
\begin{equation}
\label{eq:condprop33}
\sum_{i\in I} w_i\sum_{j\ge 1}\frac{\lambda_j^2}{\sigma_{ij}^2}<\infty.
\end{equation}
As the next proposition shows, with its proof in Appendix~\ref{app:em-kl}, this forces the initial annealed law to remain uniformly close to the target across dimensions.

\begin{proposition}
\label{prop:em-kl}
If (\ref{eq:condprop33}) holds, then $\sup_{d\ge 1}\operatorname{KL}(\rho_\star^d\,\|\,\rho_0^d)<\infty$.
\end{proposition}

The conclusions arising from Proposition~\ref{prop:em-linear-stability} are, however, \emph{scheme-dependent}. They do not rule out dimension-uniform time discretization of ALD. Rather, they show that the additional restrictions imposed by EM come from approximating the stiff diagonal linear part by a forward Euler step. This motivates the exponential-integrator discretization \cite{shi2012convergence,komori2014stochastic,zhang2022fast} studied in the next section. On each discretization interval $[t_n,t_{n+1})$, the scheme integrates the
stiff linear part of the annealed score exactly while freezing the remaining nonlinear term at its value
at the left endpoint $t_n$.

\section{Dimension-Uniform Exact-Linear-Part Discretization}\label{sec:elp}

We now define this scheme and analyze its dimension-uniform behavior.

\paragraph{Definition of the Scheme.}

Recall from the previous section that the annealed score decomposes as
\[
\nabla\log\rho_t^d(x)=-(B_t^d)^{-1}x+G_t^d(x),
\]
with a time-dependent diagonal linear part $-(B_t^d)^{-1}x$ and a nonlinear mixture correction $G_t^d(x)$, where
$
B_t^d=\operatorname{Diag}(b_{t,1},\dots,b_{t,d})
$,
$
b_{t,j}=\underline{\sigma}_j+\kappa_t\lambda_j
$.
Now fix a mesh
\[
0=t_0<t_1<\cdots<t_N=T,
\qquad
h_n:=t_{n+1}-t_n,
\qquad
h_{\max}:=\max_{0\le n\le N-1}h_n.
\]
The key idea is to freeze only the nonlinear correction over each interval $[t_n,t_{n+1})$, while integrating the resulting time-dependent diagonal linear SDE exactly. 
We call this coordinate-wise exponential-integrator discretization the exact-linear-part (ELP) scheme.
\begin{definition}
\label{def:elp}
The ELP discretization of the ALD diffusion is the sequence $(Y_n^d)_{n=0}^N$ defined by $Y_0^d\sim \rho_0^d$ and
\begin{equation}\label{eq:elp-scheme}
Y_{n+1,j}^d
=
\phi_{n,j}Y_{n,j}^d
+
\psi_{n,j}G_{t_n,j}^d(Y_n^d)
+
\xi_{n,j},
\qquad
j=1,\dots,d,
\end{equation}
where
\[
\phi_{n,j}:=\exp\!\left(-\int_{t_n}^{t_{n+1}}\frac{\gamma_j}{b_{s,j}}\,ds\right), \qquad
\psi_{n,j}
:=
\int_{t_n}^{t_{n+1}}
\exp\!\left(-\int_s^{t_{n+1}}\frac{\gamma_j}{b_{r,j}}\,dr\right)\gamma_j\,ds,
\]
and $\xi_n=(\xi_{n,1},\dots,\xi_{n,d})$ has independent centered Gaussian coordinates with
\[
\operatorname{Var}(\xi_{n,j})
=
2\gamma_j\int_{t_n}^{t_{n+1}}
\exp\!\left(-2\int_s^{t_{n+1}}\frac{\gamma_j}{b_{r,j}}\,dr\right)\,ds.
\]
\end{definition}

In the symmetric bimodal family of Proposition~\ref{prop:em-true-bias-obstruction}, one has
\[
\phi_{n,j}
=
\exp\!\left(
-\gamma_j\int_{t_n}^{t_{n+1}}\frac{dt}{\sigma_j+\kappa_t\lambda_j}
\right) = \left(
\frac{\sigma_j+\kappa_{t_{n+1}}\lambda_j}{\sigma_j+\kappa_{t_n}\lambda_j}
\right)^{\gamma_j T/\lambda_j}
\in(0,1).
\]
Thus the linear factor is contractive for every coefficient and over every time-step interval, without imposing the EM stability condition on
$\gamma_j/\underline\sigma_j$. 
This already suggests that the restrictions imposed by EM in Section~\ref{sec:em-limitations} are scheme-dependent rather than intrinsic to ALD.

\paragraph{A Dimension-Uniform KL Bound.}

We are now in the position to state the main dimension-uniform KL estimate for the ELP scheme. For this purpose, it is convenient to view the discrete scheme through its continuous-time interpolation. For $t\in[t_n,t_{n+1})$, let $Y_t^d$ solve
\begin{equation}
\label{eq:interp-raw}
dY_t^d
=
-\Gamma^d(B_t^d)^{-1}Y_t^d\,dt
+
\Gamma^d G_{t_n}^d(Y_{t_n}^d)\,dt
+
\sqrt{2\Gamma^d}\,dW_t^d,
\qquad
Y_{t_n}^d=Y_n^d.
\end{equation}
Then $Y_T^d=Y_N^d$, so $\Law(Y_T^d)$ is the terminal law of the ELP scheme.
The proof compares this interpolation with an auxiliary process
$\widehat X^d$ whose marginals satisfy $\Law(\widehat X_t^d)=\rho_t^d$; see
\eqref{eq:app-u-def}--\eqref{eq:app-reference}. By data processing and
Girsanov's formula, the terminal KL is bounded by a
$(\Gamma^d)^{-1/2}$-weighted path-space energy of the drift mismatch over $[0,T]$. Along the
auxiliary path, this mismatch splits into the transport velocity of the
annealing path, whose energy is $J_{\mathrm{ann}}^d(T)$ in
\eqref{eq:app-Jann}, and the freezing defect produced by replacing $G_t^d$ with
$G_{t_n}^d$, defined in \eqref{eq:app-freezing-defect}. This yields the KL
estimate below; the assumptions are spelled out after the theorem
and verified in Appendix~\ref{app:concrete-elp}.
% The proof compares this interpolated path with an auxiliary process
% $\widehat X^d$ whose marginals  follow the annealing path,
% \[
% \Law(\widehat X_t^d)=\rho_t^d,\qquad t\in[0,T],
% \]
% so that $\Law(\widehat X_T^d)=\rho_\star^d$. This comparison yields the KL estimate below; the proof is deferred to
% Appendix~\ref{app:concrete-elp}. The  sufficient assumptions are
% spelled out after the theorem.

\begin{theorem}
\label{thm:concrete-elp}
Assume that the  conditions
\eqref{eq:diag-step1-a}--\eqref{eq:diag-step3} below hold. Then there
exists a dimension-independent constant $C_{\mathrm{disc}}$ 
(that depends only on the bounds in \eqref{eq:diag-step1-a}--\eqref{eq:diag-step3}) 
such that
\[
\sup_{d\ge1}
\KL\!\bigl(\rho_\star^d\,\|\,\Law(Y_T^d)\bigr)
\le
\frac{1}{8T}\sum_{j\ge1}\frac{\lambda_j^2}{\gamma_j\underline{\sigma}_j}
+
C_{\mathrm{disc}}(1+T^2)h_{\max}.
\]
Thus, for every
$\varepsilon_{\mathrm{ann}},\varepsilon_{\mathrm{disc}}>0$, if
\[
T\ge
\frac{1}{8\varepsilon_{\mathrm{ann}}}
\sum_{j\ge1}\frac{\lambda_j^2}{\gamma_j\underline{\sigma}_j},
\qquad
h_{\max}\le
\frac{\varepsilon_{\mathrm{disc}}}{C_{\mathrm{disc}}(1+T^2)},
\]
then we have
\[
\sup_{d\ge1}
\KL\!\bigl(\rho_\star^d\,\|\,\Law(Y_T^d)\bigr)
\le
\varepsilon_{\mathrm{ann}}+\varepsilon_{\mathrm{disc}}.
\]
\end{theorem}

This bound separates two effects. The first term is the annealing contribution:
it decreases as the annealing time $T$ increases. The second term is the
discretization contribution: after $T$ is fixed, it can be made small by
refining the time mesh. Crucially, under the spectral assumptions \eqref{eq:diag-step1-a}--\eqref{eq:diag-step3}, this trade-off is
uniform in the truncation dimension $d$.

We now spell out the conditions \eqref{eq:diag-step1-a}--\eqref{eq:diag-step3} that constitute one set of sufficient conditions. They are organized
according to the quantities controlled in the proof
%; the corresponding estimates are proved 
given in Appendix~\ref{app:concrete-elp}.
\begin{itemize}

\item The first group controls moments of the annealing path, the linear drift
$-\Gamma^d(B_t^d)^{-1}x$, and the weighted growth of the nonlinear correction
$G_t^d$. The moment condition
\begin{equation}
\sum_{j\ge1}(\overline{\sigma}_j+\lambda_j+\overline m_j^2)<\infty
\label{eq:diag-step1-a}
\end{equation}
is used to obtain uniform moment bounds for the annealing path; see
Lemma~\ref{lem:app-moments}. The linear-drift condition
\begin{equation}
\sum_{j\ge1}\gamma_j<\infty,
\qquad
\sum_{j\ge1}\gamma_j^2
\frac{\overline{\sigma}_j+\lambda_j+\overline m_j^2}
{\underline{\sigma}_j^2}
<\infty
\label{eq:diag-step1-b}
\end{equation}
controls the fourth moment of the diagonal linear drift; see
Lemma~\ref{lem:app-linear}. Finally,
\begin{equation}
\sum_{j\ge1}\gamma_j
\frac{(\overline{\sigma}_j-\underline{\sigma}_j)^2
+\underline{\sigma}_j^2\overline m_j^2}
{\underline{\sigma}_j^4}
<\infty
\label{eq:diag-step1-c}
\end{equation}
controls the weighted growth of the nonlinear correction $G_t^d$; see
Lemma~\ref{lem:app-G-growth}. Together, these estimates also enter the
fourth-moment bound for the full auxiliary drift; see Lemma~\ref{lem:app-drift4}.

\item The second group controls the freezing defect produced by replacing the
current nonlinear field $G_t^d$ by the frozen field $G_{t_n}^d$ on each time
interval. Under the auxiliary path law, the defect is
$G_t^d(\widehat X_t^d)-G_{t_n}^d(\widehat X_{t_n}^d)$,
for $t\in[t_n,t_{n+1})$.
We split it into a spatial increment and a temporal increment:
\[
G_t^d(\widehat X_t^d)-G_{t_n}^d(\widehat X_{t_n}^d)
=
\bigl(G_t^d(\widehat X_t^d)-G_t^d(\widehat X_{t_n}^d)\bigr)
+
\bigl(G_t^d(\widehat X_{t_n}^d)-G_{t_n}^d(\widehat X_{t_n}^d)\bigr).
\]
Since $G_t^d(x)=\sum_{i\in I}p_{i,t}^d(x)c_{i,t}^d(x)$,
the estimates require controlling the spatial and temporal variation of both the
affine corrections $c_{i,t}^d$ and the responsibilities $p_{i,t}^d$.

For the spatial increment, we assume
\begin{equation}
\sum_{j\ge1}
\frac{(\overline{\sigma}_j-\underline{\sigma}_j)^2}
{\underline{\sigma}_j^4}
(\overline{\sigma}_j+\lambda_j+\overline m_j^2)
<\infty,
\qquad
\sum_{j\ge1}
\sup_{i,\ell\in I}
\frac{|m_{ij}-m_{\ell j}|^2}{\underline{\sigma}_j^2}
<\infty.
\label{eq:diag-step2-a}
\end{equation}
These conditions control the spatial variation of the affine corrections and
of the responsibilities; see Lemma~\ref{lem:app-score-diff} and
Lemma~\ref{lem:app-spatial-defect}.

For the time variation of the affine corrections, we assume
\begin{equation}
\sum_{j\ge1}
\frac{\lambda_j(\overline{\sigma}_j-\underline{\sigma}_j)}
{\underline{\sigma}_j^3}
(\overline{\sigma}_j+\lambda_j+\overline m_j^2)
<\infty,
\qquad
\sum_{j\ge1}\gamma_j
\frac{\lambda_j^2(\overline{\sigma}_j-\underline{\sigma}_j)^2}
{\underline{\sigma}_j^6}
(\overline{\sigma}_j+\lambda_j+\overline m_j^2)
<\infty.
\label{eq:diag-step2-bb}
\end{equation}
These conditions enter the bounds on the time derivatives of the affine
correction coefficients and the temporal freezing defect; see
Lemma~\ref{lem:app-ci} and Lemma~\ref{lem:app-temporal-defect}.

For the time variation of the responsibilities, we assume
\begin{equation}
\sum_{j\ge1}
\frac{\lambda_j^2}{\underline{\sigma}_j^2}
\left(
\sup_{i,\ell\in I}
\frac{|m_{ij}-m_{\ell j}|^2}{\underline{\sigma}_j^2}
+
\frac{\overline m_j^2(\overline{\sigma}_j-\underline{\sigma}_j)^2}
{\underline{\sigma}_j^4}
\right)
<\infty,
\qquad
\sum_{j\ge1}\gamma_j\lambda_j^2
\frac{\overline m_j^2}{\underline{\sigma}_j^4}
<\infty.
\label{eq:diag-step2-d}
\end{equation}
These conditions control the time derivative of the responsibilities through
the quantities $\partial_t\log\varphi_{i,t}^d$; see
Lemma~\ref{lem:app-zeta-diff} and Lemma~\ref{lem:app-temporal-defect}.

\item The final condition controls the annealing contribution. Recall
\[
J_{\mathrm{ann}}^d(T)
:=
\frac14\int_0^T\int_{\R^d}
\left\|
(\Gamma^d)^{-1/2}\frac{1}{2T}C^d\nabla\log\rho_t^d(x)
\right\|^2
\,\rho_t^d(\dd x)\,\dd t.
\]
In the diagonal Gaussian-mixture setting, this contribution is controlled by
\begin{equation}
\sum_{j\ge1}
\frac{\lambda_j^2}{\gamma_j\underline{\sigma}_j}
<\infty.
\label{eq:diag-step3}
\end{equation}
This is the estimate proved in Lemma~\ref{lem:app-transport}.

\end{itemize}
Although the conditions \eqref{eq:diag-step1-a}--\eqref{eq:diag-step3} may be hard to parse without following the proof in Appendix~\ref{app:concrete-elp}, they make explicit the two-sided role of the preconditioner discussed throughout the paper. The annealing condition
\eqref{eq:diag-step3} favors larger $\gamma_j$, while conditions
\eqref{eq:diag-step1-b}, \eqref{eq:diag-step1-c},
\eqref{eq:diag-step2-bb}, and \eqref{eq:diag-step2-d}
require enough tail damping of $\gamma_j$ to prevent discretization errors from
accumulating across high-frequency coordinates.  In simple regimes, such as
power-law spectral decay, this trade-off can be made explicit; see
Proposition~\ref{prop:power-law-balanced-preconditioner}, where balancing the annealing and discretization tails gives $\gamma_j\asymp\lambda_j^{2/3}$ for Gaussian mixtures with common-covariance tails.

\paragraph{ELP Beyond the EM Stability Restriction.}

The next proposition shows that, in contrast with what we saw for EM, the conditions ensuring a dimension-uniform bound for the ELP scheme do not force the initial smoothed law of ALD to remain uniformly close to the target across dimensions. The proof is given in Appendix~\ref{app:elp-example}.

\begin{proposition}
\label{prop:elp-example}
Let $e_1^d$ be the first canonical basis vector of $\mathbb R^d$. Fix $a>1$ and consider the
two-component target
$
\rho_\star^d
=
\frac12\mathcal N(ae_1^d,\Sigma_1^d)
+
\frac12\mathcal N(-ae_1^d,\Sigma_2^d)
$,
where
$\Sigma_1^d=\operatorname{Diag}(\sigma_j)_{j=1}^d$ and
$\Sigma_2^d=\operatorname{Diag}(\sigma_j+\delta_j)_{j=1}^d$.
Let $\sigma_j=\lambda_j=j^{-6}$ and $\gamma_j=j^{-4}$, with
$\delta_1=0$ and $\delta_j=j^{-12}$ for $j\ge2$. Then the summability
conditions \eqref{eq:diag-step1-a}--\eqref{eq:diag-step3} hold, while
\[
\KL(\rho_\star^d\,\|\,\rho_0^d)\to\infty
\qquad\text{as }d\to\infty.
\]
Nevertheless, for every $T>0$ and every mesh,
\[
\sup_{d\ge1}\KL\!\bigl(\rho_\star^d\,\|\,\Law(Y_T^d)\bigr)
\le
\frac{1}{8T}\sum_{j\ge1}j^{-2}
+
C_{\mathrm{disc}}(1+T^2)\,h_{\max}
<\infty .
\]
\end{proposition}

\section{Numerical Experiments}
\label{sec:numerics}

We conclude with a simple numerical experiment showing the high-frequency stability issue of EM identified in Section~\ref{sec:em-limitations},  and showing that ELP remains dimension-robust even when the initial smoothed law becomes increasingly far
from the target as $d$  grows.
We consider a
two-component Gaussian mixture target with common covariance
$\Sigma^d=\operatorname{Diag}(\sigma_j)$, $\sigma_j=\lambda_j=j^{-6}$, and
preconditioner $\gamma_j=j^{-4}$, initialized from
$\rho_0^d=\rho_\star^d*\mathcal N(0,2SC^d)$ with $S=5$. Then
$\KL(\rho_\star^d\|\rho_0^d)$ grows linearly with $d$, while the ELP conditions \eqref{eq:diag-step1-a}--\eqref{eq:diag-step3}
hold, so that the KL error can be made uniformly small  in $d$. For EM, however, $\gamma_j/\sigma_j=j^2$, so the stability condition of Proposition~\ref{prop:em-linear-stability} requires time steps of order $d^{-2}$; any fixed step size becomes unstable as $d$ grows.
Figure~\ref{fig:em-elp-combined} compares the two schemes. EM rapidly becomes unstable in the
high-frequency coordinates, both in KL and in the coordinate-wise variance
profile at $d=50$, whereas ELP remains close to the target scale. Further details are given in Appendix~\ref{app:numerics-details}.

\begin{figure}[H]
    \centering
    \begin{minipage}[t]{0.48\linewidth}
        \centering
        \includegraphics[width=\linewidth]{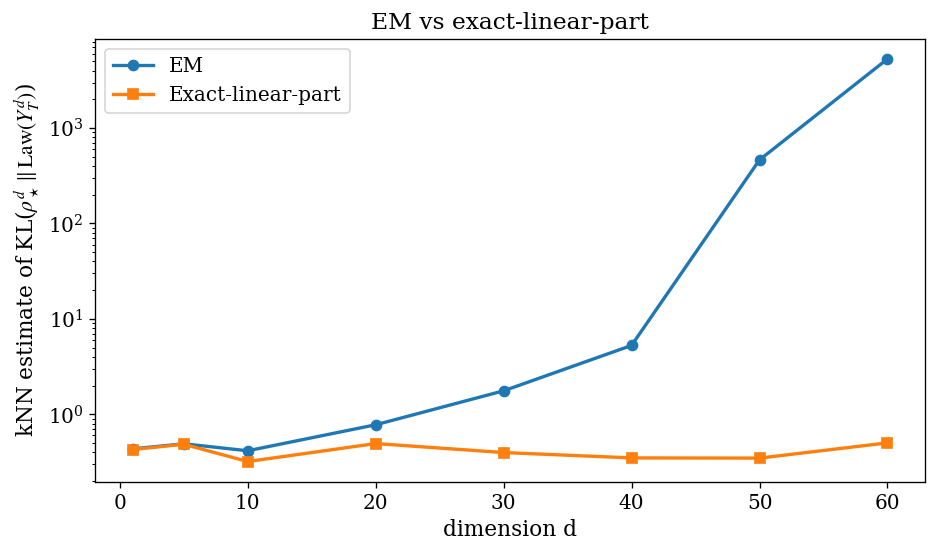}
    \end{minipage}
    \hfill
    \begin{minipage}[t]{0.48\linewidth}
        \centering
        \includegraphics[width=\linewidth]{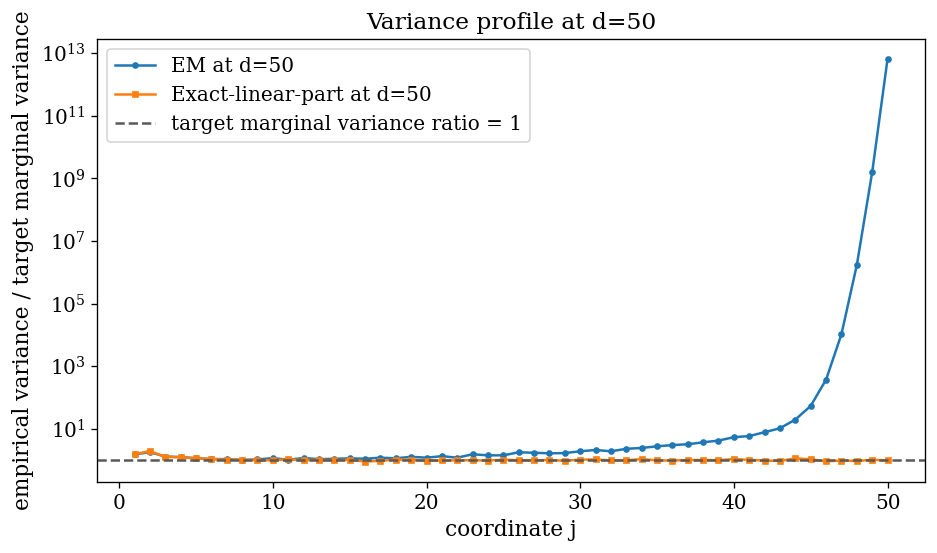}
    \end{minipage}
    \caption{Left: empirical $\KL(\rho_\star^d\|\Law(Y_T^d))$ versus dimension $d$ on
    a log scale. EM grows rapidly, while ELP remains stable. Right: coordinate-wise variance profile at $d=50$,
    normalized by the target marginal variance. EM is unstable in the
    high-frequency coordinates, while ELP remains near the target scale.  The KL is estimated by $k$NN
    with $k= 20$; robustness to $k$ is reported in Appendix~\ref{app:numerics-details}.}
    \label{fig:em-elp-combined}
\end{figure}

\section{Discussion and Future Work}

We addressed the problem of obtaining stable annealed Langevin dynamics under successive finite-dimensional approximations of an infinite-dimensional multimodal target. To the best of our knowledge, this is the first discretization analysis of ALD with guarantees that remain uniform as the finite-dimensional approximation is refined. For Gaussian mixtures,we showed that dimension-uniform control of the KL divergence from the target to the sampler law depends both on sufficient high-frequency damping by the preconditioner and a discretization that treats the high-frequency structure of the annealed score appropriately; otherwise, additional stability conditions may arise that sharply narrow the regimes in which the sampler admits dimension-uniform KL guarantees.

We analyzed two discretization schemes. For Euler--Maruyama, we showed that treating the linear part of the annealed score by a
forward Euler step imposes a stability condition which, combined with annealing-bias control, forces the smoothed law used to initialize ALD to
remain uniformly close to the target across dimensions. By
contrast, the ELP scheme of Section~\ref{sec:elp} achieves
dimension-uniform KL control even when the initial smoothed
law can be far from the target.
Together with the continuous-time analysis of
\cite{baldassari2026dimension}, this provides a unified picture of preconditioned ALD in infinite dimensions,
covering initialization mismatch, score perturbations, and discretization error.

As is common in infinite-dimensional analysis
\cite{baldassari2023conditional, pidstrigach2024infinite, baldassari2025preconditioned, baldassari2026dimension, franzese2025generative}, we worked under
structural assumptions that make the dimension-uniform conditions explicit in
terms of the relevant spectra. For example, we assumed that the mixture
covariances, smoothing covariance, and preconditioner are co-diagonalizable.
This allows these objects to be compared coordinate by coordinate, so that
dimension-uniform control can be expressed through explicit summability
conditions. Nevertheless, the model should not be viewed as simplistic: at each
fixed truncation level, mixtures with diagonal covariances can still represent a
broad class of multimodal distributions, especially when the number of
components is allowed to vary, or even be countable. With additional operator
estimates, the same strategy should extend to non-diagonal covariances and
non-diagonal preconditioners, for example under a common Loewner lower-envelope assumption on the annealed component covariances.

\bibliography{main_discretization_arxiv}
\bibliographystyle{plain}

%%%%%%%%%%%%%%%%%%%%%%%%%%%%%%%%%%%%%%%%%%%%%%%%%%%%%%%%%%%%
\newpage
\appendix

\section{Proofs of Section \ref{sec:em-limitations}}\label{app:proof-sec-3}
\subsection{Proof of Proposition \ref{prop:em-true-bias-obstruction}}\label{app:em-true-bias-obstruction}
Since the two mixture components differ only in the first coordinate, the annealed path is
\[
\pi_t^d
=
\frac12\,\mathcal N(ae_1^d,D_t^d)
+
\frac12\,\mathcal N(-ae_1^d,D_t^d),
\qquad
D_t^d=\Sigma^d+\kappa_t C^d.
\]
Its score is explicit:
\[
\nabla\log \pi_t^d(x)
=
\left(
-\frac{x_1}{d_1(t)}
+
\frac{a}{d_1(t)}
\tanh\!\left(\frac{a x_1}{d_1(t)}\right),
\;
-\frac{x_2}{d_2(t)},
\ldots,
-\frac{x_d}{d_d(t)}
\right),
\]
where
\[
d_j(t):=\sigma_j+\kappa_t\lambda_j.
\]
Hence the ALD dynamics decouples into a nonlinear first coordinate and Gaussian tail coordinates. More precisely, for every $j\ge 2$,
\[
dX_{t,j}
=
-\frac{\gamma_j}{d_j(t)}X_{t,j}\,dt
+
\sqrt{2\gamma_j}\,dW_{t,j},
\qquad
X_{0,j}\sim \mathcal N(0,\sigma_j+\lambda_j).
\]

It follows that the target and the terminal law factorize as
\[
\pi_\star^d
=
\nu_{\star,1}\otimes\bigotimes_{j=2}^d \mathcal N(0,\sigma_j),
\qquad
\operatorname{Law}(X_T^d)
=
\mu_{T,1}\otimes\bigotimes_{j=2}^d \mathcal N(0,p_j(T)),
\]
for suitable one-dimensional laws $\nu_{\star,1}$ and $\mu_{T,1}$, where $p_j(T)$ is the variance of the $j$-th coordinate at time $T$. By additivity of relative entropy,
\begin{equation}\label{eq:split-bias}
\operatorname{KL}(\pi_\star^d\,\|\,\operatorname{Law}(X_T^d))
=
\operatorname{KL}(\nu_{\star,1}\,\|\,\mu_{T,1})
+
\sum_{j=2}^d
\operatorname{KL}\!\big(\mathcal N(0,\sigma_j)\,\|\,\mathcal N(0,p_j(T))\big).
\end{equation}

We now compute the tail terms explicitly. For $j\ge 2$, the variance $p_j(t)$ solves
\[
p_j'(t)
=
2\gamma_j\left(1-\frac{p_j(t)}{d_j(t)}\right),
\qquad
p_j(0)=\sigma_j+\lambda_j.
\]
Set
\[
r_j:=\frac{\lambda_j}{\sigma_j},
\qquad
\alpha_j:=\frac{2\gamma_j T}{\lambda_j},
\qquad
\Psi(\alpha,r):=\int_1^{1+r}u^{-\alpha}\,du.
\]
A direct computation gives
\[
\frac{p_j(T)-\sigma_j}{\sigma_j}
=
\Psi(\alpha_j,r_j).
\]
Therefore
\[
\operatorname{KL}\!\big(\mathcal N(0,\sigma_j)\,\|\,\mathcal N(0,p_j(T))\big)
=
F\!\big(\Psi(\alpha_j,r_j)\big),
\]
where
\[
F(u):=\frac12\left(\log(1+u)-\frac{u}{1+u}\right),
\qquad u\ge 0.
\]

Now let
\[
M:=\sup_{j\ge 2}\frac{\gamma_j}{\sigma_j}<\infty.
\]
Then
\[
\alpha_j
=
\frac{2\gamma_j T}{\lambda_j}
=
\frac{2T}{r_j}\frac{\gamma_j}{\sigma_j}
\le
\frac{2TM}{r_j}.
\]
Since $\Psi(\alpha,r)$ is decreasing in $\alpha$, writing $c:=2TM$ yields
\[
\Psi(\alpha_j,r_j)\ge \Psi(c/r_j,r_j).
\]
Moreover,
\[
\Psi(c/r,r)
=
r\int_0^1 (1+rs)^{-c/r}\,ds.
\]
Using $\log(1+rs)\le rs$, we obtain
\[
(1+rs)^{-c/r}
=
\exp\!\left(-\frac{c}{r}\log(1+rs)\right)
\ge e^{-cs},
\]
and therefore
\[
\Psi(c/r,r)
\ge
r\int_0^1 e^{-cs}\,ds
=
\beta r,
\qquad
\beta:=\frac{1-e^{-c}}{c}>0.
\]
Thus
\[
\Psi(\alpha_j,r_j)\ge \beta r_j
\qquad\text{for all }j\ge 2.
\]

Combining this with \eqref{eq:split-bias}, the assumed uniform bound on the true annealing bias implies
\[
\sup_{d\ge 1}\sum_{j=2}^d F(\beta r_j)<\infty.
\]
If infinitely many $r_j$ were larger than $1/\beta$, then infinitely many summands would be bounded below by the positive constant $F(1)$, which is impossible. Hence $r_j\le 1/\beta$ for all but finitely many $j$. For such $j$, we have $\beta r_j\le 1$, and since
\[
F(u)=\int_0^u \frac{s}{2(1+s)^2}\,ds,
\]
it follows that for $0\le u\le 1$,
\[
F(u)\ge \frac{u^2}{16}.
\]
Hence, for all sufficiently large $j$,
\[
F(\beta r_j)\ge \frac{\beta^2}{16}r_j^2.
\]
Therefore
\[
\sum_{j\ge 2} r_j^2<\infty.
\]

Finally, the initial annealed law also factorizes:
\[
\pi_0^d
=
\nu_{0,1}\otimes\bigotimes_{j=2}^d \mathcal N(0,\sigma_j+\lambda_j),
\]
for a suitable one-dimensional law $\nu_{0,1}$. Thus
\[
\operatorname{KL}(\pi_\star^d\,\|\,\pi_0^d)
=
\operatorname{KL}(\nu_{\star,1}\,\|\,\nu_{0,1})
+
\sum_{j=2}^d F(r_j).
\]
Since $\sum_{j\ge 2}r_j^2<\infty$, we have $r_j\to 0$, so for all sufficiently large $j$ also $r_j\le 1$, and then
\[
F(r_j)\le \frac{r_j^2}{4}.
\]
Hence
\[
\sum_{j\ge 2}F(r_j)<\infty,
\]
and the first-coordinate contribution is a fixed constant independent of $d$. Therefore
\[
\sup_{d\ge 1}\operatorname{KL}(\pi_\star^d\,\|\,\pi_0^d)<\infty.
\]
This proves the claim.

\subsection{Proof of Proposition \ref{prop:em-kl}}\label{app:em-kl}

By joint convexity of relative entropy,
\[
\operatorname{KL}(\rho_\star^d\,\|\,\rho_0^d)
\le
\sum_{i\in I} w_i\,
\operatorname{KL}\!\Big(\mathcal N(m_i^d,\Sigma_i^d)\,\Big\|\,\mathcal N(m_i^d,\Sigma_i^d+C^d)\Big).
\]
Since the means coincide, the Gaussian relative entropy formula yields
\[
\operatorname{KL}(\rho_\star^d\,\|\,\rho_0^d)
\le
\frac12\sum_{i\in I} w_i\sum_{j=1}^d
\left[
\log\!\left(1+\frac{\lambda_j}{\sigma_{ij}}\right)
-
\frac{\lambda_j}{\sigma_{ij}+\lambda_j}
\right].
\]
For $r\ge 0$, the scalar function
\[
f(r):=\log(1+r)-\frac{r}{1+r}
\]
satisfies $f(r)\le r^2/2$. Therefore
\[
\operatorname{KL}(\rho_\star^d\,\|\,\rho_0^d)
\le
\frac14\sum_{i\in I} w_i\sum_{j=1}^d \frac{\lambda_j^2}{\sigma_{ij}^2}.
\]

\section{Proofs of Section \ref{sec:elp}}

\subsection{Proof of Theorem \ref{thm:concrete-elp}}\label{app:concrete-elp}

The proof has two parts. Recall that the score is decomposed as
\[
\nabla\log\rho_t^d(x)=-(B_t^d)^{-1}x+G_t^d(x),
\]
where $G_t^d$ is the nonlinear mixture correction. The ELP interpolation
$Y^d$ integrates the stiff diagonal linear part exactly, but freezes this nonlinear
correction on each discretization interval $[t_n,t_{n+1})$ at its left-endpoint value
$G_{t_n}^d(Y_{t_n}^d)$. To compare this scheme with the target, we introduce an
auxiliary process $\widehat X^d$ whose marginals exactly follow the annealing
path, that is, $\Law(\widehat X_t^d)=\rho_t^d$. When the two dynamics are
compared along the auxiliary path, the auxiliary process uses the current
nonlinear correction $G_t^d(\widehat X_t^d)$, whereas the ELP interpolation
uses the frozen correction $G_{t_n}^d(\widehat X_{t_n}^d)$. Thus the path-space KL estimate over $[0,T]$ between the law of the auxiliary process and the law of the ELP interpolation separates the error into an \emph{annealing contribution}, coming from
the transport velocity of the annealing path, and a \emph{freezing defect}
\[
G_t^d(\widehat X_t^d)-G_{t_n}^d(\widehat X_{t_n}^d),
\qquad t\in[t_n,t_{n+1}).
\]
The second part of the proof verifies that the coefficient assumptions
\eqref{eq:diag-step1-a}--\eqref{eq:diag-step3} control both contributions
uniformly in $d$, while tracking the dependence on the horizon $T$.

In this proof, we shall split the freezing defect into a \emph{spatial part}, where the time is fixed but
the spatial argument changes, and a \emph{temporal part}, where the spatial argument is
fixed but the time changes:
\[
G_t^d(\widehat X_t^d)-G_{t_n}^d(\widehat X_{t_n}^d)
=
\bigl(G_t^d(\widehat X_t^d)-G_t^d(\widehat X_{t_n}^d)\bigr)
+
\bigl(G_t^d(\widehat X_{t_n}^d)-G_{t_n}^d(\widehat X_{t_n}^d)\bigr).
\]
The estimates below are organized around this decomposition.

In what follows, for notational simplicity, we do not distinguish between
finite and countable mixtures. In the countable case, the argument can be
justified by applying the estimates to finite partial mixtures and then passing
to the limit. The normalization constants of the partial mixtures converge to
one and cancel in the responsibility formulas; moreover, the constants in the
estimates depend only on the uniform quantities appearing in the assumptions.

\paragraph{Coefficient Notation.}
We start by recording the coefficient notation used throughout. For each
$i\in I$ and $j\ge1$, we write
\[
v_{i,t,j}:=\sigma_{ij}+\kappa_t\lambda_j,
\qquad
D_{i,t}^d=\operatorname{Diag}(v_{i,t,1},\dots,v_{i,t,d}).
\]
We also define
\[
\delta\sigma_j:=\overline{\sigma}_j-\underline{\sigma}_j,
\qquad
\Delta m_j:=\sup_{i,\ell\in I}|m_{ij}-m_{\ell j}|,
\]
and
\[
A_j:=\frac{\delta\sigma_j}{\underline{\sigma}_j^2},
\qquad
B_j:=
\frac{\Delta m_j}{\underline{\sigma}_j}
+
\frac{\overline m_j\,\delta\sigma_j}{\underline{\sigma}_j^2},
\qquad
D_j:=\frac{\lambda_j\delta\sigma_j}{\underline{\sigma}_j^3}.
\]
With this notation, the summability assumptions
\eqref{eq:diag-step1-a}--\eqref{eq:diag-step3} become
\[
\sum_{j\ge1}(\overline\sigma_j+\lambda_j)<\infty,
\qquad
\sum_{j\ge1}\overline m_j^2<\infty,
\]
\[
\sum_{j\ge1}\gamma_j<\infty,
\qquad
\sum_{j\ge1}\gamma_j^2
\frac{\overline\sigma_j+\lambda_j}{\underline\sigma_j^2}<\infty,
\qquad
\sum_{j\ge1}\gamma_j^2
\frac{\overline m_j^2}{\underline\sigma_j^2}<\infty,
\]
\[
\sum_{j\ge1}\gamma_jA_j^2<\infty,
\qquad
\sum_{j\ge1}\gamma_j\frac{\overline m_j^2}{\underline\sigma_j^2}<\infty,
\]
\[
\sum_{j\ge1}A_j^2(\overline\sigma_j+\lambda_j+\overline m_j^2)<\infty,
\qquad
\sum_{j\ge1}B_j^2<\infty,
\]
\[
\sum_{j\ge1}D_j(\overline\sigma_j+\lambda_j+\overline m_j^2)<\infty,
\qquad
\sum_{j\ge1}\gamma_jD_j^2(\overline\sigma_j+\lambda_j+\overline m_j^2)<\infty,
\]
\[
\sum_{j\ge1}\lambda_j^2\frac{B_j^2}{\underline\sigma_j^2}<\infty,
\qquad
\sum_{j\ge1}\gamma_j\lambda_j^2\frac{\overline m_j^2}{\underline\sigma_j^4}<\infty,
\]
and
\[
\sum_{j\ge1}\frac{\lambda_j^2}{\gamma_j\underline\sigma_j}<\infty.
\]
Recall that
\[
G_t^d(x)=\sum_{i\in I}p_{i,t}^d(x)c_{i,t}^d(x),
\]
where $p_{i,t}^d(x)$ are the mixture responsibilities and
\[
c_{i,t}^d(x)
=
\Big((B_t^d)^{-1}-(D_{i,t}^d)^{-1}\Big)x
+
(D_{i,t}^d)^{-1}m_i^d
\]
is the affine correction associated with the $i$-th component. The terms
involving $A_j$ control the spatial variation of these affine corrections; the
terms involving $B_j$ control the spatial variation of the responsibilities;
and the terms involving $D_j$ control the time variation caused by the
annealing schedule. The final condition involving
$\lambda_j^2/(\gamma_j\underline\sigma_j)$ controls the annealing energy.

\paragraph{The Auxiliary Path and the Quantities to Be Estimated.}
On each discretization interval $[t_n,t_{n+1})$, the ELP scheme freezes $G_t^d$ at time $t_n$, and therefore does not
follow the annealing path exactly. To compare it with the target, we introduce
a reference process whose marginals are exactly $\rho_t^d$. The extra drift
below is the transport velocity associated with the evolution of the annealing
path:
\begin{equation}\label{eq:app-u-def}
u_t^d(x):=\frac{1}{2T}C^d\nabla\log\rho_t^d(x).
\end{equation}
Let $(\widehat X_t^d)_{t\in[0,T]}$ solve
\begin{equation}\label{eq:app-reference}
d\widehat X_t^d
=
-\Gamma^d(B_t^d)^{-1}\widehat X_t^d\,dt
+
\Gamma^dG_t^d(\widehat X_t^d)\,dt
+
u_t^d(\widehat X_t^d)\,dt
+
\sqrt{2\Gamma^d}\,dW_t^d,
\qquad
\widehat X_0^d\sim\rho_0^d.
\end{equation}
Since
\[
\partial_t\rho_t^d
=
-\nabla\cdot\left(
\frac{1}{2T}C^d\nabla\log\rho_t^d\,\rho_t^d
\right),
\]
and the Langevin part with drift $\Gamma^d\nabla\log\rho_t^d$ leaves
$\rho_t^d$ instantaneously invariant, the Fokker-Planck equation associated
with \eqref{eq:app-reference} is solved by $\rho_t^d$. Hence
\[
\Law(\widehat X_t^d)=\rho_t^d,
\qquad t\in[0,T],
\]
and in particular
\[
\Law(\widehat X_T^d)=\rho_T^d=\rho_\star^d.
\]

We shall use the corresponding path-matching energy
\begin{equation}\label{eq:app-Jann}
J_{\mathrm{ann}}^d(T)
:=
\frac14\int_0^T\int_{\R^d}
\left\|
(\Gamma^d)^{-1/2}u_t^d(x)
\right\|^2
\,\rho_t^d(dx)\,dt.
\end{equation}

Following \cite{guo2024provable}, which builds on
\cite{dalalyan2012sparse, chen2022sampling}, the comparison with the ELP
interpolation is based on a Girsanov estimate on path space over $[0,T]$.
Consequently, the relevant quantity is the drift mismatch between the auxiliary
process and the ELP interpolation. On an interval $[t_n,t_{n+1})$, this mismatch
is
\[
u_t^d(\widehat X_t^d)
+
\Gamma^d
\Bigl(
G_t^d(\widehat X_t^d)
-
G_{t_n}^d(\widehat X_{t_n}^d)
\Bigr).
\]
After weighting by $(\Gamma^d)^{-1/2}$, the two quantities to control are
therefore
\[
(\Gamma^d)^{-1/2}u_t^d(\widehat X_t^d)
\]
and
\begin{equation}\label{eq:app-freezing-defect}
\Delta_t^d
:=
(\Gamma^d)^{1/2}
\Bigl(
G_t^d(\widehat X_t^d)
-
G_{t_n}^d(\widehat X_{t_n}^d)
\Bigr).
\end{equation}
The first term gives the \emph{annealing contribution} $J_{\mathrm{ann}}^d(T)$. The second
term is the \emph{freezing defect}. As anticipated, we decompose it as
\[
\Delta_t^d
=
\underbrace{
(\Gamma^d)^{1/2}
\bigl(G_t^d(\widehat X_t^d)-G_t^d(\widehat X_{t_n}^d)\bigr)
}_{\text{spatial freezing defect}}
+
\underbrace{
(\Gamma^d)^{1/2}
\bigl(G_t^d(\widehat X_{t_n}^d)-G_{t_n}^d(\widehat X_{t_n}^d)\bigr)
}_{\text{temporal freezing defect}}.
\]
The proof will reduce to three estimates:
\[
J_{\mathrm{ann}}^d(T)
\lesssim
\frac1T
\sum_{j\ge1}\frac{\lambda_j^2}{\gamma_j\underline\sigma_j},
\]
\[
\E\left\|
(\Gamma^d)^{1/2}
\bigl(G_t^d(\widehat X_t^d)-G_t^d(\widehat X_s^d)\bigr)
\right\|^2
\lesssim
\left(1+T+\frac1T\right)(t-s),
\]
and
\[
\E\left\|
(\Gamma^d)^{1/2}
\bigl(G_t^d(\widehat X_s^d)-G_s^d(\widehat X_s^d)\bigr)
\right\|^2
\lesssim
\frac{(t-s)^2}{T^2}.
\]
The lemmas below prove these three estimates.

\paragraph{Moment and Transport Estimates.}
We begin with two preliminary estimates. The first gives uniform moment bounds
for the annealing path $\rho_t^d$, and hence for the marginals of
$\widehat X_t^d$. These bounds are needed later because the freezing defect
contains the random fields
\[
G_t^d(\widehat X_t^d)
\qquad\text{and}\qquad
G_t^d(\widehat X_{t_n}^d),
\]
and because the spatial freezing defect will require control of increments
\[
\widehat X_t^d-\widehat X_s^d.
\]
The second estimate controls the transport field $u_t^d$. It yields the bound
on the annealing contribution $J_{\mathrm{ann}}^d(T)$ and also provides a
fourth-moment estimate for $u_t^d(\widehat X_t^d)$, which will enter the
increment estimates below.

\begin{lemma}\label{lem:app-moments}
Under \eqref{eq:diag-step1-a}, there exist constants $M_2,M_4,M_8$,
independent of $d$, $t$, and $T$, such that
\[
\sup_{d\ge1}\sup_{t\in[0,T]}
\int_{\R^d}\|x\|^q\,\rho_t^d(dx)\le M_q,
\qquad q=2,4,8.
\]
\end{lemma}

\begin{proof}
Fix $i,d,t$. Since
\[
v_{i,t,j}=\sigma_{ij}+\kappa_t\lambda_j
\le \overline\sigma_j+\lambda_j,
\]
we have
\[
\operatorname{Tr}(D_{i,t}^d)
\le
\sum_{j\ge1}(\overline\sigma_j+\lambda_j)<\infty.
\]
Moreover,
\[
\|m_i^d\|^2
=
\sum_{j=1}^d m_{ij}^2
\le
\sum_{j\ge1}\overline m_j^2<\infty.
\]
Thus the second moments are uniformly bounded over components, $d$, and $t$.

For the fourth moment, let $Z=m+\xi$, with $\xi\sim\mathcal N(0,D)$ and $D$ diagonal. Then
\[
\|Z\|^4\le 8\|m\|^4+8\|\xi\|^4.
\]
Now
\[
\E\|\xi\|^4
=
\E\left(\sum_{j=1}^d\xi_j^2\right)^2
\le
3\left(\sum_{j=1}^d D_{jj}\right)^2
=
3(\operatorname{Tr}D)^2.
\]
Hence
\[
\mathbb E\|Z\|^4 \le 8 \|m\|^4 +24 (\operatorname{Tr}D)^2.
\]

For the eighth moment, again
\[
\|Z\|^8\le 2^7\|m\|^8+2^7\|\xi\|^8.
\]
By Minkowski's inequality in $L^4$,
\[
\left(\E\|\xi\|^8\right)^{1/4}
=
\left\|
\sum_{j=1}^d\xi_j^2
\right\|_{L^4}
\le
\sum_{j=1}^d
\|\xi_j^2\|_{L^4}
=
\sum_{j=1}^d
\left(\E|\xi_j|^8\right)^{1/4}.
\]
For a centered one-dimensional Gaussian with variance $D_{jj}$,
\[
\E|\xi_j|^8=105D_{jj}^4.
\]
Therefore
\[
\E\|\xi\|^8
\le
105(\operatorname{Tr}D)^4.
\]
It follows that 
\[
\mathbb E \|Z\|^8 \leq 2^7 \|m\|^8 + 2^7 \cdot 105(\operatorname{Tr} D)^4. 
\]
Applying these estimates to each Gaussian component $\mathcal N(m_i^d, D_{i,t}^d)$ and using the uniform bounds on $\left\|m_i^d\right\|^2$ and $\operatorname{Tr} D_{i,t}^d$ proves the required uniform fourth and eighth moment estimates. 
\end{proof}

We next control the transport field associated with the annealing path. This is
the term that produces the annealing contribution in the path-space KL estimate.

\begin{lemma}\label{lem:app-transport}
 Under $\sum_j\gamma_j<\infty$ and \eqref{eq:diag-step3}, we have
\[
\sup_{d\ge1}J_{\mathrm{ann}}^d(T)
\le
\frac{1}{16T}
\sum_{j\ge1}
\frac{\lambda_j^2}{\gamma_j\underline\sigma_j}.
\]
Moreover, there exists a constant $C_u$, independent of $d$ and $T$,
such that
\[
\sup_{d\ge1}\sup_{t\in[0,T]}
\E\|u_t^d(\widehat X_t^d)\|^4
\le
\frac{C_u}{T^4}.
\]
\end{lemma}

\begin{proof}
The argument follows that of \cite[Theorem 3.1]{baldassari2026dimension}. Write
\[
s_{i,t}^d(x):=-(D_{i,t}^d)^{-1}(x-m_i^d).
\]
Then
\[
\nabla\log\rho_t^d(x)
=
\sum_{i\in I}p_{i,t}^d(x)s_{i,t}^d(x).
\]
By Jensen's inequality,
\[
\left\|(\Gamma^d)^{-1/2}u_t^d(x)\right\|^2
\le
\frac{1}{4T^2}
\sum_{i\in I}p_{i,t}^d(x)
\left\|
(\Gamma^d)^{-1/2}C^d s_{i,t}^d(x)
\right\|^2.
\]
Integrating with respect to $\rho_t^d$ and using the mixture representation,
\begin{align*}
\int_{\R^d}
\left\|(\Gamma^d)^{-1/2}u_t^d(x)\right\|^2
\,\rho_t^d(dx)
\le
\frac{1}{4T^2}
\sum_{i\in I}w_i
\sum_{j=1}^d
\frac{\lambda_j^2}{\gamma_j}
\E_{\mathcal N(m_i^d,D_{i,t}^d)}
\left[
\frac{(X_j-m_{ij})^2}{v_{i,t,j}^2}
\right].
\end{align*}
Since
\[
\E_{\mathcal N(m_i^d,D_{i,t}^d)}
\left[
(X_j-m_{ij})^2
\right]
=
v_{i,t,j},
\]
we get
\[
\int_{\R^d}
\left\|(\Gamma^d)^{-1/2}u_t^d(x)\right\|^2
\,\rho_t^d(dx)
\le
\frac{1}{4T^2}
\sum_{j=1}^d
\frac{\lambda_j^2}{\gamma_j\underline\sigma_j}.
\]
Integrating in time and multiplying by $1/4$ according to \eqref{eq:app-Jann} gives the desired bound on
$J_{\mathrm{ann}}^d(T)$.

For the fourth moment, Jensen's inequality gives
\[
\|u_t^d(x)\|^4
\le
\frac{1}{16T^4}
\sum_{i\in I}p_{i,t}^d(x)
\|C^d s_{i,t}^d(x)\|^4.
\]
Integrating against $\rho_t^d$ and using the mixture decomposition,
\[
\E\|u_t^d(\widehat X_t^d)\|^4
\le
\frac{1}{16T^4}
\sum_{i\in I}w_i
\E_{\mathcal N(m_i^d,D_{i,t}^d)}
\|C^d s_{i,t}^d(X)\|^4.
\]
Under the $i$-th component $\mathcal N(m_i^d,D_{i,t}^d)$, the coordinates of
$C^d s_{i,t}^d(X)$ are independent centered Gaussians with variances
$\lambda_j^2/v_{i,t,j}$. Hence
\[
\E_{\mathcal N(m_i^d,D_{i,t}^d)}
\|C^d s_{i,t}^d(X)\|^4=
\Bigl(\sum_{j=1}^d\frac{\lambda_j^2}{v_{i,t,j}}\Bigr)^2
+2\sum_{j=1}^d \frac{\lambda_j^4}{v_{i,t,j}^2}\le
3\left(
\sum_{j=1}^d\frac{\lambda_j^2}{v_{i,t,j}}
\right)^2
\le
3\left(
\sum_{j\ge1}\frac{\lambda_j^2}{\underline\sigma_j}
\right)^2.
\]
Under \eqref{eq:diag-step3},
\[
\sum_{j\ge1}\frac{\lambda_j^2}{\underline\sigma_j}
\le
\left(\sup_{j\ge1}\gamma_j\right)
\sum_{j\ge1}\frac{\lambda_j^2}{\gamma_j\underline\sigma_j}
<\infty,
\]
because $\sum_j\gamma_j<\infty$ implies $\sup_j\gamma_j<\infty$. Thus the
fourth-moment estimate holds with
\[
C_u:=
\frac{3}{16}
\left(
\sum_{j\ge1}\frac{\lambda_j^2}{\underline\sigma_j}
\right)^2,
\]
which is independent of $d$ and $T$.
\end{proof}

\paragraph{Drift and Increment Estimates.}
The next estimates are used to control the spatial freezing defect. We
will compare
\[
G_t^d(\widehat X_t^d)
\qquad\text{and}\qquad
G_t^d(\widehat X_s^d).
\]
The spatial variation of $G_t^d$ is controlled by a Lipschitz factor with
polynomial growth:
\[
\|(\Gamma^d)^{1/2}(G_t^d(x)-G_t^d(y))\|
\lesssim
L_t^d(x,y)\|x-y\|.
\]
After substituting $x=\widehat X_t^d$ and $y=\widehat X_s^d$, the spatial
freezing estimate therefore involves a product of the form
\[
\E\Big[
L_t^d(\widehat X_t^d,\widehat X_s^d)^2
\,
\|\widehat X_t^d-\widehat X_s^d\|^2
\Big].
\]
The moment bounds for the annealing path control the first factor, while
H\"older's inequality leaves an $L^4$ norm of the increment. This is why we need
a fourth-moment estimate for the auxiliary path:
\[
\E\|\widehat X_t^d-\widehat X_s^d\|^4
\le
C\left(1+T^2+\frac1{T^2}\right)(t-s)^2.
\]

To obtain this increment estimate, we use the SDE for the auxiliary process.
Writing its drift as
\[
b_t^d(x)
=
-\Gamma^d(B_t^d)^{-1}x
+
\Gamma^dG_t^d(x)
+
u_t^d(x),
\]
we have
\[
\widehat X_t^d-\widehat X_s^d
=
\int_s^t b_r^d(\widehat X_r^d)\,dr
+
\sqrt{2\Gamma^d}(W_t^d-W_s^d).
\]
Therefore the required fourth-moment increment bound follows from two contributions:
a fourth-moment bound on the drift $b_t^d(\widehat X_t^d)$ and the standard
fourth-moment bound for the Brownian increment. 
The lemmas below establish
these contributions in order: first the linear part of the drift, then the nonlinear correction,
then the full auxiliary drift, and finally the increment estimate.

We begin with the diagonal linear part of the drift.

\begin{lemma}\label{lem:app-linear}
Under \eqref{eq:diag-step1-b}, there exists a constant
$B_{\mathrm{lin},4}$, independent of $d$ and $T$, such that
\[
\sup_{d\ge1}\sup_{t\in[0,T]}
\E\left\|
\Gamma^d(B_t^d)^{-1}\widehat X_t^d
\right\|^4
\le B_{\mathrm{lin},4}.
\]
\end{lemma}

\begin{proof}
Let
\[
A_t^d:=\Gamma^d(B_t^d)^{-1}
=
\operatorname{Diag}\left(\frac{\gamma_1}{b_{t,1}},\dots,
\frac{\gamma_d}{b_{t,d}}\right).
\]
Under the component $\mathcal N(m_i^d,D_{i,t}^d)$, the random vector $A_t^dX$
is Gaussian with mean $A_t^d m_i^d$ and diagonal covariance with entries
\[
\frac{\gamma_j^2}{b_{t,j}^2}v_{i,t,j}.
\]
Since $b_{t,j}\ge\underline\sigma_j$ and
$v_{i,t,j}\le\overline\sigma_j+\lambda_j$, we have
\[
\|A_t^d m_i^d\|^2
\le
\sum_{j\ge1}
\gamma_j^2
\frac{\overline m_j^2}{\underline\sigma_j^2}<\infty
\]
and
\[
\operatorname{Tr}
\left[
\operatorname{Cov}(A_t^dX)
\right]
\le
\sum_{j\ge1}
\gamma_j^2
\frac{\overline\sigma_j+\lambda_j}{\underline\sigma_j^2}<\infty.
\]
Using the Gaussian fourth-moment estimate from Lemma~\ref{lem:app-moments} and
averaging over the mixture proves the result.
\end{proof}

Next we record the coordinatewise bounds on the affine corrections
$c_{i,t}^d$, which will be used both for the growth of $G_t^d$ and later for
the temporal freezing defect.

\begin{lemma}\label{lem:app-ci}
For every $i,j$,
\[
c_{i,t,j}^d(x)=\alpha_{ij}(t)x_j+\beta_{ij}(t),
\]
where
\[
\alpha_{ij}(t):=
\frac{\sigma_{ij}-\underline\sigma_j}{b_{t,j}v_{i,t,j}},
\qquad
\beta_{ij}(t):=
\frac{m_{ij}}{v_{i,t,j}}.
\]
Moreover,
\[
0\le\alpha_{ij}(t)\le A_j,
\qquad
|\beta_{ij}(t)|
\le
\frac{\overline m_j}{\underline\sigma_j}.
\]
Finally,
\[
|\dot\alpha_{ij}(t)|
\le
\frac{2D_j}{T},
\qquad
|\dot\beta_{ij}(t)|
\le
\frac{\lambda_j\overline m_j}{T\underline\sigma_j^2}.
\]
\end{lemma}

\begin{proof}
The formulas for $\alpha_{ij}$ and $\beta_{ij}$ follow directly from the
definition of $c_{i,t}^d$. Since $b_{t,j},v_{i,t,j}\ge\underline\sigma_j$,
\[
0\le
\alpha_{ij}(t)
=
\frac{\sigma_{ij}-\underline\sigma_j}{b_{t,j}v_{i,t,j}}
\le
\frac{\delta\sigma_j}{\underline\sigma_j^2}
=
A_j, \qquad 
|\beta_{ij}(t)|
\le
\frac{\overline m_j}{\underline\sigma_j}.
\]
Differentiating and using
\[
\dot b_{t,j}=\dot v_{i,t,j}=-\frac{\lambda_j}{T}
\]
yields
\[
\dot\alpha_{ij}(t)
=
\frac{\lambda_j(\sigma_{ij}-\underline\sigma_j)}{T}
\left(
\frac{1}{b_{t,j}^2v_{i,t,j}}
+
\frac{1}{b_{t,j}v_{i,t,j}^2}
\right),
\]
and hence
\[
|\dot\alpha_{ij}(t)|
\le
\frac{2\lambda_j\delta\sigma_j}{T\underline\sigma_j^3}
=
\frac{2D_j}{T}.
\]
Similarly,
\[
\dot\beta_{ij}(t)
=
\frac{\lambda_j m_{ij}}{Tv_{i,t,j}^2},
\]
which gives the desired bound.
\end{proof}

These coordinatewise bounds imply the following weighted growth estimate for
the nonlinear correction $G_t^d$.

\begin{lemma}\label{lem:app-G-growth}
Under \eqref{eq:diag-step1-c}, there exist constants $G_0,G_1$,
independent of $d$ and $T$, such that for every $d,t,x$,
\[
\|(\Gamma^d)^{1/2}G_t^d(x)\|
\le
G_0+G_1\|x\|.
\]
\end{lemma}

\begin{proof}
Set
\[
C_{A,\gamma}:=\sum_{j\ge1}\gamma_jA_j^2,
\qquad
C_{m,\gamma}:=
\sum_{j\ge1}
\gamma_j\frac{\overline m_j^2}{\underline\sigma_j^2},
\]
which are finite under \eqref{eq:diag-step1-c}.
Using convexity of the norm and $G_t^d=\sum_i p_{i,t}^d c_{i,t}^d$,
\[
\|(\Gamma^d)^{1/2}G_t^d(x)\|
\le
\sum_{i\in I}p_{i,t}^d(x)
\|(\Gamma^d)^{1/2}c_{i,t}^d(x)\|.
\]
By Lemma~\ref{lem:app-ci},
\[
\|(\Gamma^d)^{1/2}c_{i,t}^d(x)\|^2
\le
2\sum_{j=1}^d\gamma_jA_j^2x_j^2
+
2\sum_{j=1}^d
\gamma_j\frac{\overline m_j^2}{\underline\sigma_j^2}.
\]
Therefore
\[
\|(\Gamma^d)^{1/2}c_{i,t}^d(x)\|^2
\le
2C_{A,\gamma}\|x\|^2+2C_{m,\gamma}.
\]
Taking square roots gives 
\[
\|(\Gamma^d)^{1/2}c_{i,t}^d(x)\|
\le \sqrt{2C_{m,\gamma}}+
\sqrt{2C_{A,\gamma}}\|x\|.
\]
\end{proof}

Combining the linear drift, the nonlinear correction, and the transport drift
$u_t^d$ gives the required fourth-moment bound for the full auxiliary drift.
\begin{lemma}\label{lem:app-drift4}
Define
\[
b_t^d(x):=
-\Gamma^d(B_t^d)^{-1}x
+
\Gamma^dG_t^d(x)
+
u_t^d(x).
\]
Under \eqref{eq:diag-step1-a}, \eqref{eq:diag-step1-b},
\eqref{eq:diag-step1-c}, and \eqref{eq:diag-step3}, there exist finite
constants $B_{\mathrm{dr},4}^{(0)}$ and $B_{\mathrm{dr},4}^{(1)}$, independent
of $d$ and $T$, such that
\[
\sup_{d\ge1}\sup_{t\in[0,T]}
\E\|b_t^d(\widehat X_t^d)\|^4
\le
B_{\mathrm{dr},4}(T),
\qquad
B_{\mathrm{dr},4}(T):=
B_{\mathrm{dr},4}^{(0)}+\frac{B_{\mathrm{dr},4}^{(1)}}{T^4}.
\]
\end{lemma}

\begin{proof}
By Lemma~\ref{lem:app-linear},
\[
\sup_{d,t}
\E
\|\Gamma^d(B_t^d)^{-1}\widehat X_t^d\|^4
\le B_{\mathrm{lin},4}.
\]
Since $\sum_j\gamma_j<\infty$, also $\gamma_\infty:=\sup_j\gamma_j<\infty$.
Using Lemma~\ref{lem:app-G-growth},
\[
\|\Gamma^dG_t^d(x)\|^2
=
\sum_{j=1}^d\gamma_j^2G_{t,j}^d(x)^2
\le
\gamma_\infty
\sum_{j=1}^d\gamma_jG_{t,j}^d(x)^2
=
\gamma_\infty
\|(\Gamma^d)^{1/2}G_t^d(x)\|^2.
\]
Therefore
\[
\|\Gamma^dG_t^d(x)\|^4
\le
\gamma_\infty^2(G_0+G_1\|x\|)^4
\le
8\gamma_\infty^2(G_0^4+G_1^4\|x\|^4).
\]
Taking expectation under $\Law(\widehat X_t^d)=\rho_t^d$ and using
Lemma~\ref{lem:app-moments} gives
\[
\sup_{d,t}
\E\|\Gamma^dG_t^d(\widehat X_t^d)\|^4
\le
8\gamma_\infty^2(G_0^4+G_1^4M_4),
\]
which is independent of $T$. By Lemma~\ref{lem:app-transport},
\[
\sup_{d,t}\E\|u_t^d(\widehat X_t^d)\|^4
\le
\frac{C_u}{T^4}.
\]
The result follows from
\[
\|a+b+c\|^4\le 27(\|a\|^4+\|b\|^4+\|c\|^4),
\]
with
\[
B_{\mathrm{dr},4}^{(0)}
:=
27\left[
B_{\mathrm{lin},4}
+
8\gamma_\infty^2(G_0^4+G_1^4M_4)
\right],
\qquad
B_{\mathrm{dr},4}^{(1)}:=27C_u.
\]
\end{proof}

We can now turn the drift bound into the desired increment estimate.
\begin{lemma}\label{lem:app-increment4}
Under the assumptions of Lemma~\ref{lem:app-drift4}, there exists a 
constant $C_{\mathrm{inc},4}$, independent of $d$ and $T$, such that
\[
\E\|\widehat X_t^d-\widehat X_s^d\|^4
\le
C_{\mathrm{inc},4}
\left(1+T^2+\frac1{T^2}\right)
(t-s)^2,
\qquad
0\le s\le t\le T.
\]
\end{lemma}

\begin{proof}
From \eqref{eq:app-reference},
\[
\widehat X_t^d-\widehat X_s^d
=
\int_s^t b_r^d(\widehat X_r^d)\,dr
+
\sqrt{2\Gamma^d}(W_t^d-W_s^d).
\]
Thus
\[
\E\|\widehat X_t^d-\widehat X_s^d\|^4
\le
8\E\left\|\int_s^t b_r^d(\widehat X_r^d)\,dr\right\|^4
+
8\E\|\sqrt{2\Gamma^d}(W_t^d-W_s^d)\|^4.
\]
By Jensen's inequality,
\[
\left\|\int_s^t z_r\,dr\right\|^4
\le
(t-s)^3\int_s^t\|z_r\|^4\,dr.
\]
Using Lemma~\ref{lem:app-drift4}, we obtain
\[
\E\left\|\int_s^t b_r^d(\widehat X_r^d)\,dr\right\|^4
\le
\left(
B_{\mathrm{dr},4}^{(0)}
+
\frac{B_{\mathrm{dr},4}^{(1)}}{T^4}
\right)(t-s)^4.
\]
Since $t-s\le T$,
\[
\left(
B_{\mathrm{dr},4}^{(0)}
+
\frac{B_{\mathrm{dr},4}^{(1)}}{T^4}
\right)(t-s)^4
\le
\left(
B_{\mathrm{dr},4}^{(0)}T^2
+
\frac{B_{\mathrm{dr},4}^{(1)}}{T^2}
\right)(t-s)^2.
\]
The Brownian increment is centered Gaussian with covariance
$2(t-s)\Gamma^d$. Hence
\[
\E\|\sqrt{2\Gamma^d}(W_t^d-W_s^d)\|^4
=
4(t-s)^2
\left[
\operatorname{Tr}(\Gamma^d)^2
+
2\operatorname{Tr}((\Gamma^d)^2)
\right].
\]
Since
\[
\operatorname{Tr}(\Gamma^d)\le\sum_{j\ge1}\gamma_j<\infty,
\qquad
\operatorname{Tr}((\Gamma^d)^2)\le\sum_{j\ge1}\gamma_j^2<\infty,
\]
we obtain
\[
\E\|\widehat X_t^d-\widehat X_s^d\|^4
\le
C_{\mathrm{inc},4}
\left(1+T^2+\frac1{T^2}\right)(t-s)^2,
\]
with $C_{\mathrm{inc},4}$ being a constant independent of $d$ and $T$.
\end{proof}

\paragraph{Spatial Freezing Defect.}
We now estimate
\[
(\Gamma^d)^{1/2}
\bigl(G_t^d(\widehat X_t^d)-G_t^d(\widehat X_s^d)\bigr).
\]
Since
\[
G_t^d(x)=\sum_{i\in I}p_{i,t}^d(x)c_{i,t}^d(x),
\]
its spatial variation has two contributions: an \emph{affine-correction increment}
\[
\sum_{i\in I}p_{i,t}^d(x)
\bigl(c_{i,t}^d(x)-c_{i,t}^d(y)\bigr)
\]
and a \emph{responsibility increment}
\[
\sum_{i\in I}
\bigl(p_{i,t}^d(x)-p_{i,t}^d(y)\bigr)c_{i,t}^d(y).
\]
The first contribution is controlled by the affine coefficients $A_j$. The
second requires a Lipschitz estimate for the responsibilities $p_{i,t}^d$. We
derive this estimate from pairwise differences of the component scores
$s_{i,t}^d-s_{\ell,t}^d$ in the next lemma.

\begin{lemma}\label{lem:app-score-diff}
For every $i,\ell,j$,
\[
s_{i,t,j}^d(x)-s_{\ell,t,j}^d(x)
=
-a_{i\ell,j}(t)x_j+b_{i\ell,j}(t),
\]
where
\[
a_{i\ell,j}(t):=
\frac{1}{v_{i,t,j}}-\frac{1}{v_{\ell,t,j}},
\qquad
b_{i\ell,j}(t):=
\frac{m_{ij}}{v_{i,t,j}}
-
\frac{m_{\ell j}}{v_{\ell,t,j}}.
\]
Moreover,
\[
|a_{i\ell,j}(t)|\le A_j,
\qquad
|b_{i\ell,j}(t)|\le B_j.
\]
Hence, with
\[
\mathcal A^d(x):=
\left(\sum_{j=1}^dA_j^2x_j^2\right)^{1/2},
\qquad
B_0:=
\left(2\sum_{j\ge1}B_j^2\right)^{1/2},
\]
we have, uniformly in $i,\ell,d,t,x$,
\[
\|s_{i,t}^d(x)-s_{\ell,t}^d(x)\|
\le
\sqrt{2}\,\mathcal A^d(x)+B_0.
\]
\end{lemma}

\begin{proof}
The affine formula is immediate from the definition of the component scores $s_{i,t}^d$. Also, since
$v_{i,t,j},v_{\ell,t,j}\ge\underline\sigma_j$,
\[
|a_{i\ell,j}(t)|
=
\left|
\frac{1}{v_{i,t,j}}-\frac{1}{v_{\ell,t,j}}
\right|
=
\frac{|\sigma_{\ell j}-\sigma_{ij}|}{v_{i,t,j}v_{\ell,t,j}}
\le
\frac{\delta\sigma_j}{\underline\sigma_j^2}
=
A_j.
\]
Similarly,
\[
|b_{i\ell,j}(t)|
\le
\frac{|m_{ij}-m_{\ell j}|}{v_{i,t,j}}
+
|m_{\ell j}|
\left|
\frac{1}{v_{i,t,j}}-\frac{1}{v_{\ell,t,j}}
\right|
\le
\frac{\Delta m_j}{\underline\sigma_j}
+
\frac{\overline m_j\delta\sigma_j}{\underline\sigma_j^2}
=
B_j.
\]
The norm bound follows by expanding the square, using
\[
(a-b)^2\le 2a^2+2b^2,
\]
and applying the bounds on $a_{i\ell,j}(t)$ and $b_{i\ell,j}(t)$:
\[
\|s_{i,t}^d(x) - s_{\ell,t}^d(x)\|^2 \leq 2 \sum_{j=1}^d A_j^2 x_j^2 + 2 \sum_{j=1}^d B_j^2 \leq 2(\mathcal A^d(x))^2 + B_0^2.
\]
\end{proof}

The previous lemma is the missing ingredient for the responsibility part of the
spatial variation. Together with the affine bounds in Lemma~\ref{lem:app-ci}, it
allows us to estimate
\[
G_t^d(\widehat X_t^d)-G_t^d(\widehat X_s^d).
\]
The increment bound from Lemma~\ref{lem:app-increment4} then converts this
pointwise spatial estimate into a bound of order
$\left(1+T+T^{-1}\right)(t-s)$.

\begin{lemma}\label{lem:app-spatial-defect}
Under \eqref{eq:diag-step1-a}, \eqref{eq:diag-step1-c}, and
\eqref{eq:diag-step2-a}, there exists a constant $C_{\mathrm{sp}}$,
independent of $d$ and $T$, such that for all $0\le s\le t\le T$,
\[
\E\left\|
(\Gamma^d)^{1/2}
\bigl(G_t^d(\widehat X_t^d)-G_t^d(\widehat X_s^d)\bigr)
\right\|^2
\le
C_{\mathrm{sp}}
\left(1+T+\frac1T\right)
(t-s).
\]
\end{lemma}

\begin{proof}
Set
\[
C_{A,\gamma}:=\sum_{j\ge1}\gamma_jA_j^2<\infty,
\qquad
C_{c,0}:=
\left(
\sum_{j\ge1}\gamma_j
\frac{\overline m_j^2}{\underline\sigma_j^2}
\right)^{1/2},
\qquad
C_{c,1}:=(2C_{A,\gamma})^{1/2}.
\]
Fix $d,t$ and $x,y\in\R^d$. Decompose
\[
G_t^d(x)-G_t^d(y)
=
\sum_{i\in I}p_{i,t}^d(x)\bigl(c_{i,t}^d(x)-c_{i,t}^d(y)\bigr)
+
\sum_{i\in I}
\bigl(p_{i,t}^d(x)-p_{i,t}^d(y)\bigr)c_{i,t}^d(y).
\]

For the affine-correction increment, convexity and Lemma~\ref{lem:app-ci} give
\begin{align*}
\left\|
(\Gamma^d)^{1/2}
\sum_{i\in I}p_{i,t}^d(x)\bigl(c_{i,t}^d(x)-c_{i,t}^d(y)\bigr)
\right\|^2
& \le
\sum_{i\in I}p_{i,t}^d(x)
\left\|
(\Gamma^d)^{1/2}
\bigl(c_{i,t}^d(x)-c_{i,t}^d(y)\bigr)
\right\|^2
\\
&\le
\sum_{j=1}^d\gamma_jA_j^2(x_j-y_j)^2.
\end{align*}
Hence, taking $x=\widehat X_t^d$ and $y=\widehat X_s^d$,
\[
\E
\sum_{j=1}^d\gamma_jA_j^2
(\widehat X_{t,j}^d-\widehat X_{s,j}^d)^2
\le
C_{A,\gamma}\E\|\widehat X_t^d-\widehat X_s^d\|^2.
\]
By Lemma~\ref{lem:app-increment4},
\[
\E\|\widehat X_t^d-\widehat X_s^d\|^2
\le
\left(\E\|\widehat X_t^d-\widehat X_s^d\|^4\right)^{1/2}
\le
C\left(1+T+\frac1T\right)(t-s),
\]
which gives the desired spatial bound for the affine-correction increment.

We next control the responsibility increment. By Lemma~\ref{lem:app-ci},
\[
\sup_{i\in I}
\|(\Gamma^d)^{1/2}c_{i,t}^d(y)\|
\le
\sqrt{2}\,C_{c,0}+C_{c,1}\|y\|.
\]
Moreover,
\[
\nabla p_{i,t}^d(z)
=
p_{i,t}^d(z)
\left(
s_{i,t}^d(z)
-
\sum_{\ell\in I}p_{\ell,t}^d(z)s_{\ell,t}^d(z)
\right).
\]
Therefore, using Lemma~\ref{lem:app-score-diff} and convexity,
\[
\sum_{i\in I}\|\nabla p_{i,t}^d(z)\|
\le
\sum_{i,\ell\in I}
p_{i,t}^d(z)p_{\ell,t}^d(z)
\|s_{i,t}^d(z)-s_{\ell,t}^d(z)\|
\le
\sqrt{2}\,\mathcal A^d(z)+B_0.
\]
The mean-value theorem along the segment
$z_\theta=(1-\theta)y+\theta x$ gives
\begin{align*}
\sum_{i\in I}|p_{i,t}^d(x)-p_{i,t}^d(y)|
&\le
\|x-y\|
\int_0^1
\sum_{i\in I}\|\nabla p_{i,t}^d(z_\theta)\|\,d\theta
\\
&\le
\bigl(
\sqrt{2}\,\mathcal A^d(x)
+
\sqrt{2}\,\mathcal A^d(y)
+
B_0
\bigr)
\|x-y\|.
\end{align*}
Hence
\begin{align*}
&\left\|
(\Gamma^d)^{1/2}
\sum_{i\in I}
\bigl(p_{i,t}^d(x)-p_{i,t}^d(y)\bigr)c_{i,t}^d(y)
\right\|
\\
&\qquad\le
\bigl(\sqrt{2}\,C_{c,0}+C_{c,1}\|y\|\bigr)
\bigl(
\sqrt{2}\,\mathcal A^d(x)
+
\sqrt{2}\,\mathcal A^d(y)
+
B_0
\bigr)
\|x-y\|.
\end{align*}
Now set $x=\widehat X_t^d$ and $y=\widehat X_s^d$. H\"older's inequality with
exponents $(4,4,2)$ yields
\begin{align*}
&\E\left\|
(\Gamma^d)^{1/2}
\sum_{i\in I}
\bigl(p_{i,t}^d(\widehat X_t^d)-p_{i,t}^d(\widehat X_s^d)\bigr)
c_{i,t}^d(\widehat X_s^d)
\right\|^2
\\
&\qquad\le
\left\|\sqrt{2}\,C_{c,0}+C_{c,1}\|\widehat X_s^d\|\right\|_{L^8}^2
\left\|
\sqrt{2}\,\mathcal A^d(\widehat X_t^d)
+
\sqrt{2}\,\mathcal A^d(\widehat X_s^d)
+
B_0
\right\|_{L^8}^2
\left\|
\widehat X_t^d-\widehat X_s^d
\right\|_{L^4}^2.
\end{align*}
The first factor is uniformly bounded  by Lemma~\ref{lem:app-moments}, with no
dependence on $T$. To bound the second factor, note that
\[
\bigl(\mathcal A^d(x)\bigr)^2
=
\sum_{j=1}^dA_j^2x_j^2.
\]
By Minkowski's inequality in $L^4$,
\[
\left(\E\bigl[\mathcal A^d(\widehat X_r^d)\bigr]^8\right)^{1/4}
\le
\sum_{j=1}^d
A_j^2
\left(\E|\widehat X_{r,j}^d|^8\right)^{1/4}.
\]
Under $\Law(\widehat X_r^d)=\rho_r^d$, the $j$-th coordinate is a
one-dimensional Gaussian mixture with component means bounded by $\overline m_j$
and variances bounded by $\overline\sigma_j+\lambda_j$. Hence there is a constant $C>0$ such that
\[
\left(\E|\widehat X_{r,j}^d|^8\right)^{1/4}
\le
C(\overline\sigma_j+\lambda_j+\overline m_j^2).
\]
By \eqref{eq:diag-step2-a},
\[
\sum_{j\ge1}
A_j^2(\overline\sigma_j+\lambda_j+\overline m_j^2)<\infty.
\]
Therefore the second factor is uniformly bounded, again with no dependence on
$T$. Finally, Lemma~\ref{lem:app-increment4} gives
\[
\left\|
\widehat X_t^d-\widehat X_s^d
\right\|_{L^4}^2
\le
C \left(1+T+\frac1T\right)(t-s).
\]
Combining the affine and responsibility estimates proves the result for some constant $C_{\mathrm{sp}}$.
\end{proof}

\paragraph{Temporal Freezing Defect.}
We next estimate
\[
(\Gamma^d)^{1/2}
\bigl(G_t^d(\widehat X_s^d)-G_s^d(\widehat X_s^d)\bigr),
\]
where the spatial variable is fixed and only the time argument changes. Using
\[
G_t^d(x)=\sum_{i\in I}p_{i,t}^d(x)c_{i,t}^d(x),
\]
we split the temporal variation into two contributions:
\[
\sum_{i\in I}p_{i,t}^d(x)
\bigl(c_{i,t}^d(x)-c_{i,s}^d(x)\bigr),
\]
the affine-correction increment, and
\[
\sum_{i\in I}
\bigl(p_{i,t}^d(x)-p_{i,s}^d(x)\bigr)c_{i,s}^d(x),
\]
the responsibility increment. The first term is controlled by the time
derivatives of the affine-correction coefficients, namely
$\dot\alpha_{ij}$ and $\dot\beta_{ij}$. The second term requires a bound on
the time variation of the responsibilities, which is obtained through
\[
\zeta_{i,t}^d(x):=\partial_t\log\varphi_{i,t}^d(x).
\]
Since $\kappa_t=(T-t)/T$, each time derivative contributes a factor $1/T$.
This is the mechanism behind the final bound of order $(t-s)^2/T^2$.

\begin{lemma}\label{lem:app-zeta-diff}
Define
\[
\zeta_{i,t}^d(x):=\partial_t\log\varphi_{i,t}^d(x),
\qquad
\mathcal D^d(x):=\sum_{j=1}^dD_jx_j^2.
\]
Under \eqref{eq:diag-step1-a}, \eqref{eq:diag-step2-bb}, and
\eqref{eq:diag-step2-d}, there exist constants $H_0,H_1$, independent
of $d$ and $T$, such that, uniformly in $i,\ell,d,t,x$,
\[
|\zeta_{i,t}^d(x)-\zeta_{\ell,t}^d(x)|
\le
\frac{1}{T}\left(
H_0+H_1\|x\|+\mathcal D^d(x)
\right).
\]
\end{lemma}

\begin{proof}
A direct differentiation gives
\[
\zeta_{i,t}^d(x)
=
\frac{1}{2T}
\sum_{j=1}^d
\lambda_j
\left[
\frac{1}{v_{i,t,j}}
-
\frac{(x_j-m_{ij})^2}{v_{i,t,j}^2}
\right].
\]
Therefore
\[
\zeta_{i,t}^d(x)-\zeta_{\ell,t}^d(x)
=
\frac{1}{2T}
\sum_{j=1}^d\lambda_j
\left[
\left(
\frac{1}{v_{i,t,j}}-\frac{1}{v_{\ell,t,j}}
\right)
-
\left(
\frac{(x_j-m_{ij})^2}{v_{i,t,j}^2}
-
\frac{(x_j-m_{\ell j})^2}{v_{\ell,t,j}^2}
\right)
\right].
\]
We expand the quadratic difference:
\begin{align*}
\frac{(x_j-m_{ij})^2}{v_{i,t,j}^2}
-
\frac{(x_j-m_{\ell j})^2}{v_{\ell,t,j}^2}
=
\left(
\frac{1}{v_{i,t,j}^2}
-
\frac{1}{v_{\ell,t,j}^2}
\right)x_j^2
-2\left(
\frac{m_{ij}}{v_{i,t,j}^2}
-
\frac{m_{\ell j}}{v_{\ell,t,j}^2}
\right)x_j
+
\left(
\frac{m_{ij}^2}{v_{i,t,j}^2}
-
\frac{m_{\ell j}^2}{v_{\ell,t,j}^2}
\right).
\end{align*}
The coefficient of $x_j^2$ is bounded by
\[
\left|
\frac{1}{v_{i,t,j}^2}
-
\frac{1}{v_{\ell,t,j}^2}
\right|
\le
\frac{2\delta\sigma_j}{\underline\sigma_j^3},
\]
and hence, after multiplication by $\lambda_j$, by $2D_j$.

For the linear coefficient,
\[
\left|
\frac{m_{ij}}{v_{i,t,j}^2}
-
\frac{m_{\ell j}}{v_{\ell,t,j}^2}
\right|
\le
\frac{\Delta m_j}{\underline\sigma_j^2}
+
\frac{2\overline m_j\delta\sigma_j}{\underline\sigma_j^3}
\le
\frac{2B_j}{\underline\sigma_j}.
\]
Thus
\[
\sum_{j=1}^d
\lambda_j
\left|
\frac{m_{ij}}{v_{i,t,j}^2}
-
\frac{m_{\ell j}}{v_{\ell,t,j}^2}
\right|
|x_j|
\le
2
\left(
\sum_{j\ge1}
\lambda_j^2\frac{B_j^2}{\underline\sigma_j^2}
\right)^{1/2}
\|x\|.
\]
For the constant terms, using
\[
|m_{ij}^2-m_{\ell j}^2|
\le
2\overline m_j\Delta m_j,
\]
we get
\[
\left|
\frac{m_{ij}^2}{v_{i,t,j}^2}
-
\frac{m_{\ell j}^2}{v_{\ell,t,j}^2}
\right|
\le
\frac{2\overline m_j\Delta m_j}{\underline\sigma_j^2}
+
\frac{2\overline m_j^2\delta\sigma_j}{\underline\sigma_j^3}
\le
\frac{2\overline m_jB_j}{\underline\sigma_j}.
\]
Hence
\[
\sum_{j\ge1}
\lambda_j
\frac{\overline m_jB_j}{\underline\sigma_j}
\le
\left(
\sum_{j\ge1}\overline m_j^2
\right)^{1/2}
\left(
\sum_{j\ge1}
\lambda_j^2\frac{B_j^2}{\underline\sigma_j^2}
\right)^{1/2}
<\infty.
\]
Finally,
\[
\sum_{j\ge1}\lambda_jA_j
=
\sum_{j\ge1}D_j\underline\sigma_j
\le
\sum_{j\ge1}D_j\overline\sigma_j
<\infty
\]
by \eqref{eq:diag-step2-bb}. Collecting these bounds proves the claim, with the
only dependence on $T$ given by the explicit prefactor $1/T$.
\end{proof}

We now combine the affine-correction estimate with the preceding
time-derivative bound for the responsibilities.

\begin{lemma}\label{lem:app-temporal-defect}
Under \eqref{eq:diag-step1-a}, \eqref{eq:diag-step1-c},
\eqref{eq:diag-step2-bb}, and \eqref{eq:diag-step2-d}, there exists a
constant $C_{\mathrm{tm}}$, independent of $d$ and $T$, such that for all
$0\le s\le t\le T$,
\[
\E\left\|
(\Gamma^d)^{1/2}
\bigl(G_t^d(\widehat X_s^d)-G_s^d(\widehat X_s^d)\bigr)
\right\|^2
\le
\frac{C_{\mathrm{tm}}}{T^2}(t-s)^2.
\]
\end{lemma}

\begin{proof}
Set
\[
C_{A,\gamma}:=\sum_{j\ge1}\gamma_jA_j^2<\infty,
\qquad
C_{c,1}:=(2C_{A,\gamma})^{1/2},
\qquad
C_{c,0}:=
\left(
\sum_{j\ge1}\gamma_j
\frac{\overline m_j^2}{\underline\sigma_j^2}
\right)^{1/2}.
\]
Fix $d$ and $0\le s\le t\le T$. For $x\in\R^d$, we write
\[
G_t^d(x)-G_s^d(x)
=
\sum_{i\in I}p_{i,t}^d(x)\bigl(c_{i,t}^d(x)-c_{i,s}^d(x)\bigr)
+
\sum_{i\in I}
\bigl(p_{i,t}^d(x)-p_{i,s}^d(x)\bigr)c_{i,s}^d(x).
\]
We treat the two terms separately.

We first bound the affine-correction increment. By convexity and
Lemma~\ref{lem:app-ci},
\begin{align*}
&\left\|
(\Gamma^d)^{1/2}
\sum_{i\in I}p_{i,t}^d(x)\bigl(c_{i,t}^d(x)-c_{i,s}^d(x)\bigr)
\right\|^2
\\
&\qquad\le
\sup_{i\in I}
\left\|
(\Gamma^d)^{1/2}
\bigl(c_{i,t}^d(x)-c_{i,s}^d(x)\bigr)
\right\|^2
\\
&\qquad\le
2\sup_{i\in I}
\sum_{j=1}^d
\gamma_j
|\alpha_{ij}(t)-\alpha_{ij}(s)|^2x_j^2
+
2\sup_{i\in I}
\sum_{j=1}^d
\gamma_j
|\beta_{ij}(t)-\beta_{ij}(s)|^2.
\end{align*}
By the mean-value theorem and Lemma~\ref{lem:app-ci},
\[
|\alpha_{ij}(t)-\alpha_{ij}(s)|
\le
\frac{2(t-s)}{T}D_j,
\qquad
|\beta_{ij}(t)-\beta_{ij}(s)|
\le
\frac{t-s}{T}
\frac{\lambda_j\overline m_j}{\underline\sigma_j^2}.
\]
Hence, for $x = \widehat X_s^d$,
\begin{align*}
&\E\left\|
(\Gamma^d)^{1/2}
\sum_{i\in I}p_{i,t}^d(\widehat X_s^d)
\bigl(c_{i,t}^d(\widehat X_s^d)-c_{i,s}^d(\widehat X_s^d)\bigr)
\right\|^2
\\
&\qquad\le
\frac{C(t-s)^2}{T^2}
\left[
\sum_{j\ge1}
\gamma_jD_j^2
(\overline\sigma_j+\lambda_j+\overline m_j^2)
+
\sum_{j\ge1}
\gamma_j\lambda_j^2
\frac{\overline m_j^2}{\underline\sigma_j^4}
\right],
\end{align*}
for some constant $C>0$, independent of $d$, $s$, $t$, and $T$.
The bracketed quantity is finite by \eqref{eq:diag-step2-bb} and
\eqref{eq:diag-step2-d}. Therefore the affine-correction increment is bounded by a constant multiple of
$(t-s)^2/T^2$, uniformly in $d$.

We now bound the responsibility increment. Since
\[
\partial_r p_{i,r}^d(x)
=
p_{i,r}^d(x)
\left(
\zeta_{i,r}^d(x)
-
\sum_{\ell\in I}p_{\ell,r}^d(x)\zeta_{\ell,r}^d(x)
\right),
\]
Lemma~\ref{lem:app-zeta-diff} gives
\[
\sum_{i\in I}|\partial_r p_{i,r}^d(x)|
\le
\frac{1}{T}
\left(
H_0+H_1\|x\|+\mathcal D^d(x)
\right).
\]
Hence
\[
\sum_{i\in I}|p_{i,t}^d(x)-p_{i,s}^d(x)|
\le \int_s^t \sum_{i \in I} \left|\partial_r p_{i,r}^d(x) \right| \, \text{d}r \le
\frac{t-s}{T}
\left(
H_0+H_1\|x\|+\mathcal D^d(x)
\right).
\]
Moreover, Lemma~\ref{lem:app-ci} gives
\[
\sup_{i\in I}\|(\Gamma^d)^{1/2}c_{i,s}^d(x)\|
\le
\sqrt{2}\,C_{c,0}+C_{c,1}\|x\|.
\]
Therefore
\begin{align*}
&\left\|
(\Gamma^d)^{1/2}
\sum_{i\in I}
\bigl(p_{i,t}^d(x)-p_{i,s}^d(x)\bigr)c_{i,s}^d(x)
\right\|
\\
&\qquad\le \sup_{i\in I}\norm{(\Gamma^d)^{1/2}c_{i,s}^d(x)}\sum_{i\in I}\abs{p_{i,t}^d(x)-p_{i,s}^d(x)}\\&\qquad\le
\frac{t-s}{T}
\bigl(\sqrt{2}\,C_{c,0}+C_{c,1}\|x\|\bigr)
\bigl(H_0+H_1\|x\|+\mathcal D^d(x)\bigr).
\end{align*}
Now set $x=\widehat X_s^d$. By Hölder,
\begin{align*}
&\E\Bigl\|(\Gamma^d)^{1/2}\sum_{i\in I} (p_{i,t}^d(\widehat X_s^d)-p_{i,s}^d(\widehat X_s^d))c_{i,s}^d(\widehat X_s^d)\Bigr\|^2\\
&\qquad\le
\frac{(t-s)^2}{T^2}
\Bigl\|\sqrt{2}\,C_{c,0}+C_{c,1}\norm{\widehat X_s^d}\Bigr\|_{L^4}^2
\Bigl\|H_0+H_1\norm{\widehat X_s^d}+\mathcal D^d(\widehat X_s^d)\Bigr\|_{L^4}^2.
\end{align*}

By Lemma~\ref{lem:app-moments}, the
moments of $\|\widehat X_s^d\|$ needed here are uniformly bounded, with no
dependence on $T$. It remains only to bound $\mathcal D^d(\widehat X_s^d)$ in
$L^4$. Since
\[
\mathcal D^d(x)
=
\sum_{j=1}^dD_jx_j^2,
\]
Minkowski's inequality in $L^4$ gives
\[
\left\|D^d(\widehat X_s^d)\right\|_{L^4} =\left(\E|\mathcal D^d(\widehat X_s^d)|^4\right)^{1/4}
\le
\sum_{j=1}^d
D_j
\left(\E|\widehat X_{s,j}^d|^8\right)^{1/4}.
\]
As in the proof of Lemma~\ref{lem:app-spatial-defect},
\[
\left(\E|\widehat X_{s,j}^d|^8\right)^{1/4}
\le
C(\overline\sigma_j+\lambda_j+\overline m_j^2),
\]
and the summability condition \eqref{eq:diag-step2-bb} gives
\[
\sum_{j\ge1}
D_j(\overline\sigma_j+\lambda_j+\overline m_j^2)<\infty.
\]
Thus the responsibility increment is also bounded by a constant multiple of $(t-s)^2/T^2$,
uniformly in $d$. Combining the affine and responsibility bounds proves the
result.
\end{proof}

\paragraph{Path-Space KL Estimate and Conclusion.}
We now combine the preceding estimates. Let $(\Omega,\mathcal F,(\mathcal F_t)_{t\in[0,T]},P)$ carry the unique strong solution $\widehat X^d$ to \eqref{eq:app-reference}, together with the driving $d$-dimensional Brownian motion $W^d$. Let $P^d:=P\circ(\widehat X^d)^{-1}$ be the path law of $\widehat X^d$ on $C([0,T];\R^d)$, and let $Q^d$ be the path law of the interpolation process $Y^d$ defined by \eqref{eq:interp-raw}. The terminal marginal of $P^d$ is $\rho_\star^d$, whereas
the terminal marginal of $Q^d$ is $\Law(Y_T^d)$. Thus a path-space relative
entropy estimate, followed by data processing under the evaluation map at time $T$, gives
the desired terminal KL bound.

\begin{proof}[Proof of Theorem~\ref{thm:concrete-elp}]
Fix $d\ge1$. By construction, the marginal at time $T$ of $Q^d$ is
$\Law(Y_T^d)$, while the marginal at time $T$ of $P^d$ is $\rho_\star^d$.

The two path laws have the same diffusion coefficient $\sqrt{2\Gamma^d}$ and
differ only in their drifts. Since the ELP interpolation uses a frozen value at
the previous mesh point, it is useful to write the two drifts as adapted
functionals on path space. For a continuous path
$\omega\in C([0,T];\mathbb R^d)$ and $t\in[t_n,t_{n+1})$, set
\[
b_t^{P,d}(\omega)
:=
-\Gamma^d(B_t^d)^{-1}\omega_t
+
\Gamma^dG_t^d(\omega_t)
+
u_t^d(\omega_t),
\]
and
\[
b_t^{Q,d}(\omega)
:=
-\Gamma^d(B_t^d)^{-1}\omega_t
+
\Gamma^dG_{t_n}^d(\omega_{t_n}).
\]
The second drift is not Markovian in the current state alone, but it is
progressively measurable: on $[t_n,t_{n+1})$ it depends only on the current
value $\omega_t$ and on the previously observed mesh value $\omega_{t_n}$.
Thus it is admissible in the localized path-space Girsanov argument below.
Evaluating these functionals along the reference path $\widehat X^d$, we obtain
\[
(\Gamma^d)^{-1/2}
\bigl(b_t^{P,d}(\widehat X^d)-b_t^{Q,d}(\widehat X^d)\bigr)
=
(\Gamma^d)^{-1/2}u_t^d(\widehat X_t^d)
+
\Delta_t^d,
\]
where
\[
\Delta_t^d
=
(\Gamma^d)^{1/2}
\Bigl(
G_t^d(\widehat X_t^d)
-
G_{t_n}^d(\widehat X_{t_n}^d)
\Bigr).
\]

The square integrability needed for Girsanov follows from
Lemma~\ref{lem:app-transport}, Lemma~\ref{lem:app-G-growth}, and
Lemma~\ref{lem:app-moments}. Indeed,
\[
\norm{\Delta_t^d}^2
\le
2\norm{(\Gamma^d)^{1/2}G_t^d(\widehat X_t^d)}^2
+
2\norm{(\Gamma^d)^{1/2}G_{t_n}^d(\widehat X_{t_n}^d)}^2.
\]
By the growth bound on $G_t^d$,
\[
\norm{(\Gamma^d)^{1/2}G_t^d(x)}
\le
G_0+G_1\norm{x},
\]
and therefore
\[
\sup_{d\ge1}\sup_{t\in[0,T]}
\E\norm{(\Gamma^d)^{1/2}G_t^d(\widehat X_t^d)}^2
\le
2G_0^2+2G_1^2M_2
<\infty.
\]
The same estimate applies to
$G_{t_n}^d(\widehat X_{t_n}^d)$. Hence, for each fixed $T$,
\[
\E\int_0^T \norm{\Delta_t^d}^2\,\dd t<\infty.
\]
Also,
\[
\E\int_0^T \norm{(\Gamma^d)^{-1/2}u_t^d(\widehat X_t^d)}^2\,\dd t
=4J_{\mathrm{ann}}^d(T)<\infty.
\]
Therefore
\[
\E\int_0^T
\norm{
(\Gamma^d)^{-1/2}
\bigl(b_t^{P,d}(\widehat X^d)-b_t^{Q,d}(\widehat X^d)\bigr)
}^2\,\dd t
<\infty.
\]

Set
\[
\beta_t^d
:=
\frac{1}{\sqrt2}
(\Gamma^d)^{-1/2}
\bigl(b_t^{P,d}(\widehat X^d)-b_t^{Q,d}(\widehat X^d)\bigr).
\]
Then
\[
\E\int_0^T \norm{\beta_t^d}^2\,\dd t<\infty.
\]
For $m\ge1$, define the stopping time
\[
\tau_m^d
:=
\inf\left\{
t\in[0,T]:
\int_0^t\|\beta_s^d\|^2\,ds\ge m
\right\}\wedge T.
\]
Then
\[
\int_0^{\tau_m^d}\|\beta_s^d\|^2\,ds\le m
\qquad P\text{-a.s.}
\]
The stopped stochastic exponential
\[
Z_t^{d,m}
:=
\exp\left(
-\int_0^{t\wedge\tau_m^d}\beta_s^d\cdot dW_s^d
-
\frac12
\int_0^{t\wedge\tau_m^d}\|\beta_s^d\|^2\,ds
\right)
\]
is therefore a true $P$-martingale on $[0,T]$. Define a probability measure
$Q^{d,m}$ on $(\Omega, \mathcal F)$ by
\[
\frac{dQ^{d,m}}{dP}:=Z_T^{d,m}.
\]
By the stopped Girsanov theorem,
\[
W_t^{d,m}:=
W_t^d+\int_0^{t\wedge\tau_m^d}\beta_s^d\,ds
\]
is a Brownian motion under $Q^{d,m}$, and for $t\le\tau_m^d$,
\[
d\widehat X_t^d
=
b_t^{Q,d}(\widehat X^d)\,dt+\sqrt{2\Gamma^d}\,dW_t^{d,m}.
\]
That is, up to the stopping time $\tau_m^d$, the process $\widehat X^d$ solves
the same path-dependent equation as the ELP interpolation, with drift
$b_t^{Q,d}$ and diffusion coefficient $\sqrt{2\Gamma^d}$. This equation has a
unique strong solution, obtained recursively on the mesh intervals: once the
value at $t_n$ is known, the drift on $[t_n,t_{n+1})$ is the sum of a
time-dependent linear term in the current state and the frozen vector
$\Gamma^dG_{t_n}^d(Y_{t_n}^d)$. The same recursive construction gives
uniqueness in law for the stopped equation.

Now let
\[
S_m:C([0,T];\mathbb R^d)\to C([0,T];\mathbb R^d),
\qquad
S_m(\omega):=(\omega_{t\wedge\tau_m(\omega)})_{t\in[0,T]},
\]
where $\tau_m(\omega)$ denotes the canonical version of the stopping time above.
Write
\[
P_m^d:=P^d\circ S_m^{-1},
\qquad
Q_m^d:=Q^d\circ S_m^{-1}.
\]
Therefore the law of $\widehat X_{\cdot\wedge\tau_m^d}^d$ under $Q^{d,m}$
coincides with $Q_m^d$.

By data processing under the measurable map
$\widehat X^d\mapsto \widehat X_{\cdot\wedge\tau_m^d}^d$,
\[
\KL(P_m^d\,\|\,Q_m^d)
\le
\KL(P\,\|\,Q^{d,m}).
\]
Moreover,
\begin{align*}
\KL(P\,\|\,Q^{d,m})
&=
\E_P\log\frac{dP}{dQ^{d,m}}
=
-\E_P\log Z_T^{d,m}
\\
&=
\frac12
\E_P\int_0^{\tau_m^d}\|\beta_t^d\|^2\,dt
\\
&=
\frac14
\E_P\int_0^{\tau_m^d}
\left\|
(\Gamma^d)^{-1/2}
\bigl(b_t^{P,d}(\widehat X^d)-b_t^{Q,d}(\widehat X^d)\bigr)
\right\|^2\,dt.
\end{align*}
Using $(a+b)^2\le 2a^2+2b^2$ and the definition of
$J_{\mathrm{ann}}^d(T)$, we get
\[
\KL(P_m^d\,\|\,Q_m^d)
\le
2J_{\mathrm{ann}}^d(T)
+
\frac12
\sum_{n=0}^{N-1}
\int_{t_n}^{t_{n+1}}
\E\left[
\mathbf 1_{\{t\le\tau_m^d\}}\|\Delta_t^d\|^2
\right]\,dt.
\]
Since $\tau_m^d\uparrow T$ almost surely, the stopped observations increase to
the full path observation. Hence, by monotonicity of relative entropy under
increasing observations, followed by monotone convergence on the right-hand
side,
\begin{equation}\label{eq:app-path-kl-split}
\KL(P^d\,\|\,Q^d)
\le
2J_{\mathrm{ann}}^d(T)
+
\frac12
\sum_{n=0}^{N-1}
\int_{t_n}^{t_{n+1}}
\E\|\Delta_t^d\|^2\,dt.
\end{equation}
Finally, data processing under the evaluation map at time $T$ gives
\[
\KL\bigl(\rho_\star^d\,\|\,\Law(Y_T^d)\bigr)
\le
\KL(P^d\,\|\,Q^d).
\]

It remains to estimate the two terms on the right-hand side of
\eqref{eq:app-path-kl-split}. By Lemma~\ref{lem:app-transport},
\[
2J_{\mathrm{ann}}^d(T)
\le
\frac{1}{8T}
\sum_{j=1}^d
\frac{\lambda_j^2}{\gamma_j\underline\sigma_j}
\le
\frac{1}{8T}
\sum_{j\ge1}
\frac{\lambda_j^2}{\gamma_j\underline\sigma_j}.
\]

For the freezing term, fix $n$ and $t\in[t_n,t_{n+1})$. Using the decomposition
of $\Delta_t^d$ into its spatial and temporal parts,
\[
\Delta_t^d
=
(\Gamma^d)^{1/2}
\left(
G_t^d(\widehat X_t^d)-G_t^d(\widehat X_{t_n}^d)
\right)
+
(\Gamma^d)^{1/2}
\left(
G_t^d(\widehat X_{t_n}^d)-G_{t_n}^d(\widehat X_{t_n}^d)
\right),
\]
together with
$(a+b)^2\le2a^2+2b^2$, Lemma~\ref{lem:app-spatial-defect}, and
Lemma~\ref{lem:app-temporal-defect}, we obtain
\[
\E\|\Delta_t^d\|^2
\le
2C_{\mathrm{sp}}
\left(1+T+\frac1T\right)(t-t_n)
+
\frac{2C_{\mathrm{tm}}}{T^2}(t-t_n)^2.
\]
Integrating over $[t_n,t_{n+1}]$ gives
\[
\int_{t_n}^{t_{n+1}}
\E\|\Delta_t^d\|^2\,dt
\le
C_{\mathrm{sp}}
\left(1+T+\frac1T\right)h_n^2
+
\frac{2C_{\mathrm{tm}}}{3T^2}h_n^3.
\]
Summing over $n$ and using
\[
\sum_{n=0}^{N-1}h_n^2\le T h_{\max},
\qquad
\sum_{n=0}^{N-1}h_n^3\le T^2 h_{\max},
\]
we find
\[
\sum_{n=0}^{N-1}
\int_{t_n}^{t_{n+1}}
\E\|\Delta_t^d\|^2\,dt
\le
\left[
C_{\mathrm{sp}}(1+T+T^2)
+
\frac{2C_{\mathrm{tm}}}{3}
\right]h_{\max}.
\]
Inserting this bound into \eqref{eq:app-path-kl-split} yields
\[
\KL\bigl(\rho_\star^d\,\|\,\Law(Y_T^d)\bigr)
\le
\frac{1}{8T}
\sum_{j\ge1}
\frac{\lambda_j^2}{\gamma_j\underline\sigma_j}
+
C_{\mathrm{disc},T}h_{\max},
\]
where
\[
C_{\mathrm{disc},T}
:=
\frac12 C_{\mathrm{sp}}(1+T+T^2)
+
\frac13 C_{\mathrm{tm}}.
\]
In particular, there exists a constant $C_{\mathrm{disc}}$, depending only on
the summability bounds, such that
\[
C_{\mathrm{disc},T}
\le
C_{\mathrm{disc}}(1+T+T^2).
\]
The constants do not depend on $d$ or on the mesh. Taking the supremum over $d$
proves Theorem \ref{thm:concrete-elp}. The final
$\varepsilon_{\mathrm{ann}}+\varepsilon_{\mathrm{disc}}$ statement follows
immediately from the estimate above. Since the discretization constant grows at
most like $1+T+T^2$, it is enough to impose
\[
h_{\max}
\le
\frac{\varepsilon_{\mathrm{disc}}}
{C_{\mathrm{disc}}(1+T+T^2)}
\]
after choosing $T$ so that the annealing contribution is at most
$\varepsilon_{\mathrm{ann}}$.
\end{proof}

\subsection{Proof of Proposition \ref{prop:elp-example}}\label{app:elp-example}

Let
\[
\rho_\star^d
=
\frac12 \mathcal N(m_1^d,\Sigma_1^d)
+
\frac12 \mathcal N(m_2^d,\Sigma_2^d),
\qquad
\rho_0^d
=
\rho_\star^d * \mathcal N(0,C^d).
\]
In the example,
\[
\sigma_{1j}=\sigma_j=j^{-6},
\qquad
\sigma_{2j}=\sigma_j+\delta_j,
\qquad
\delta_1=0,\quad \delta_j=j^{-12}\quad (j\ge2).
\]
Hence
\[
\underline\sigma_j=\sigma_j=j^{-6},
\qquad
\overline\sigma_j=\sigma_j+\delta_j,
\qquad
\overline\sigma_j-\underline\sigma_j=\delta_j.
\]
Moreover,
\[
\overline m_1=a,\qquad \overline m_j=0\quad (j\ge2).
\]
Define 
\[
\Delta m_j:= \sup_{i,\ell\in I}|m_{ij}-m_{\ell j}|.
\]
Hence 
\[
\Delta m_1 =2a,
\qquad
\Delta m_j =0\quad (j\ge2).
\]

We first verify the summability assumptions
\eqref{eq:diag-step1-a}--\eqref{eq:diag-step3}. Since
\[
\lambda_j=j^{-6},
\qquad
\gamma_j=j^{-4},
\qquad
\delta_j=j^{-12}\quad (j\ge2),
\]
we have
\[
\sum_{j\ge1}(\overline\sigma_j+\lambda_j)
\lesssim
\sum_{j\ge1}j^{-6}<\infty,
\qquad
\sum_{j\ge1}\overline m_j^2=a^2<\infty.
\]
Also,
\[
\sum_{j\ge1}\gamma_j
=
\sum_{j\ge1}j^{-4}<\infty,
\]
and
\[
\sum_{j\ge1}
\gamma_j^2
\frac{\overline\sigma_j+\lambda_j}{\underline\sigma_j^2}
\lesssim
\sum_{j\ge1}
j^{-8}\frac{j^{-6}}{j^{-12}}
=
\sum_{j\ge1}j^{-2}
<\infty.
\]
The term involving the means is finite because $\overline m_j=0$ for all
$j\ge2$, while the first coordinate contributes only a constant:
\[
\sum_{j\ge1}
\gamma_j^2\frac{\overline m_j^2}{\underline\sigma_j^2}
=
\gamma_1^2\frac{a^2}{\sigma_1^2}
<\infty.
\]
Next,
\[
\sum_{j\ge1}
\gamma_j\frac{(\overline\sigma_j-\underline\sigma_j)^2}
{\underline\sigma_j^4}
=
\sum_{j\ge2}
j^{-4}\frac{j^{-24}}{j^{-24}}
<\infty,
\]
and again
\[
\sum_{j\ge1}\gamma_j\frac{\overline m_j^2}{\underline\sigma_j^2}
<\infty
\]
because only the first coordinate contributes.

We now check the conditions controlling the freezing defect. First,
\[
\sum_{j\ge1}
\frac{(\overline\sigma_j-\underline\sigma_j)^2}
{\underline\sigma_j^4}
(\overline\sigma_j+\lambda_j+\overline m_j^2)
\lesssim
\sum_{j\ge2}
\frac{j^{-24}}{j^{-24}}\,j^{-6} = \sum_{j\ge2}
j^{-6} 
<\infty.
\]
Furthermore,
\[
\sum_{j\ge1}
\left(
\frac{\Delta m_j}{\underline\sigma_j}
+
\frac{\overline m_j(\overline\sigma_j-\underline\sigma_j)}
{\underline\sigma_j^2}
\right)^2
<\infty,
\]
because the summand vanishes for all $j\ge2$ and the first coordinate is finite.

For the time-variation terms, since $\delta_j=0$ at $j=1$ and
$\overline m_j=0$ for $j\ge2$,
\[
\sum_{j\ge1}
\frac{\lambda_j(\overline\sigma_j-\underline\sigma_j)}
{\underline\sigma_j^3}
(\overline\sigma_j+\lambda_j+\overline m_j^2)
\lesssim
\sum_{j\ge2}
\frac{j^{-6}j^{-12}}{j^{-18}}\,j^{-6}
=
\sum_{j\ge2}j^{-6}
<\infty,
\]
and
\[
\sum_{j\ge1}
\gamma_j
\frac{\lambda_j^2(\overline\sigma_j-\underline\sigma_j)^2}
{\underline\sigma_j^6}
(\overline\sigma_j+\lambda_j+\overline m_j^2)
\lesssim
\sum_{j\ge2}
j^{-4}
\frac{j^{-12}j^{-24}}{j^{-36}}\,j^{-6}
=
\sum_{j\ge2}j^{-10}
<\infty.
\]
Similarly,
\[
\sum_{j\ge1}
\frac{\lambda_j^2}{\underline\sigma_j^2}
\left(
\frac{\Delta m_j}{\underline\sigma_j}
+
\frac{\overline m_j(\overline\sigma_j-\underline\sigma_j)}
{\underline\sigma_j^2}
\right)^2
<\infty,
\]
again because only the first coordinate contributes. Also,
\[
\sum_{j\ge1}
\gamma_j\lambda_j^2
\frac{\overline m_j^2}{\underline\sigma_j^4}
<\infty
\]
for the same reason. Finally,
\[
\sum_{j\ge1}
\frac{\lambda_j^2}{\gamma_j\underline\sigma_j}
=
\sum_{j\ge1}
\frac{j^{-12}}{j^{-4}j^{-6}}
=
\sum_{j\ge1}j^{-2}
<\infty.
\]
Thus all assumptions \eqref{eq:diag-step1-a}--\eqref{eq:diag-step3} hold.

We next prove that
\[
\KL(\rho_\star^d\,\|\,\rho_0^d)\to\infty.
\]
We use the variational characterization of relative entropy. For $t>0$, define
\[
S_d(x):=\sum_{j=2}^d \frac{x_j^2}{\sigma_j}.
\]
Then
\[
\KL(\rho_\star^d\,\|\,\rho_0^d)
\ge
-t\,\mathbb E_{\rho_\star^d}[S_d]
-
\log \mathbb E_{\rho_0^d}\!\left[e^{-tS_d}\right].
\]
Under $\rho_\star^d$, the tail coordinates have, conditionally on the mixture component, variances
\[
\sigma_j
\quad\text{or}\quad
\sigma_j+\delta_j.
\]
Therefore
\[
\mathbb E_{\rho_\star^d}[S_d]
=
\frac12\sum_{j=2}^d 1
+
\frac12\sum_{j=2}^d\left(1+\frac{\delta_j}{\sigma_j}\right)
=
(d-1)+\frac12\sum_{j=2}^d j^{-6}
\le
(d-1)+C_0,
\]
where $C_0:=\frac12\sum_{j\ge2}j^{-6}<\infty$.

On the other hand, under $\rho_0^d$, the two tail covariance profiles are
\[
\sigma_j+\lambda_j=2\sigma_j,
\qquad
\sigma_j+\delta_j+\lambda_j=2\sigma_j+\delta_j.
\]
Hence
\[
\mathbb E_{\rho_0^d}\!\left[e^{-tS_d}\right]
=
\frac12\prod_{j=2}^d (1+4t)^{-1/2}
+
\frac12\prod_{j=2}^d
\left(1+2t\left(2+\frac{\delta_j}{\sigma_j}\right)\right)^{-1/2}.
\]
Since $\delta_j/\sigma_j=j^{-6}\ge0$, the second product is bounded above by the first one. Therefore
\[
\mathbb E_{\rho_0^d}\!\left[e^{-tS_d}\right]
\le
(1+4t)^{-(d-1)/2}.
\]
Hence,
\[
\KL(\rho_\star^d\,\|\,\rho_0^d)
\ge
(d-1)\left(\frac12\log(1+4t)-t\right)
-tC_0.
\]
Choosing, for instance, $t=1/8$, we get
\[
\frac12\log(1+4t)-t
=
\frac12\log\!\left(\frac32\right)-\frac18
>0.
\]
Thus
\[
\KL(\rho_\star^d\,\|\,\rho_0^d)\to\infty
\qquad\text{as }d\to\infty.
\]

Finally, since the summability assumptions
\eqref{eq:diag-step1-a}--\eqref{eq:diag-step3} have been verified, Theorem~\ref{thm:concrete-elp} applies. In the example,
\[
\sum_{j\ge1}
\frac{\lambda_j^2}{\gamma_j\underline\sigma_j}
=
\sum_{j\ge1}j^{-2}.
\]
Therefore, for every $T>0$ and every mesh,
\[
\sup_{d\ge1}
\KL\!\bigl(\rho_\star^d\,\|\,\Law(Y_T^d)\bigr)
\le
\frac{1}{8T}\sum_{j\ge1}j^{-2}
+
C_{\mathrm{disc}}(1+T^2)\,h_{\max}
<\infty.
\]

\section{Preconditioner Design for the ELP Scheme in the Power-Law Regime}\label{app:power-law}

We use a simple power-law regime to extract a concrete design rule for the ELP
scheme from the sufficient conditions of Theorem~\ref{thm:concrete-elp}. In the
simplified setting below, these conditions reveal a clear trade-off.

The role of the preconditioner $(\gamma_j)$ is two-sided. On the one hand, larger
coefficients $\gamma_j$ accelerate the annealing dynamics and improve the
annealing contribution. On the other hand, if $\gamma_j$ is too large in the
high-frequency tail, discretization errors can accumulate across coordinates.
Thus the preconditioner must be large enough to make annealing efficient, but
sufficiently damped in the tail to control the discretization error.

The main takeaway of this section is the following. In the common-covariance
tail regime considered below, these two roles of the preconditioner are balanced by the choice
\[
\gamma_j\asymp \lambda_j^{2/3}.
\]
With this choice, the dominant annealing and discretization contributions have
the same per-coordinate order,
\[
\frac{\lambda_j^{4/3}}{\sigma_j}.
\]
Thus, in this simplified regime, the sufficient conditions of
Theorem~\ref{thm:concrete-elp} reduce to
\[
\sum_{j\ge1}\frac{\lambda_j^{4/3}}{\sigma_j}<\infty.
\]
For power-law spectra
\[
\sigma_j\asymp j^{-a},
\qquad
\lambda_j\asymp j^{-b}, \qquad b\ge a,
\]
this becomes
\[
b>\frac{3(a+1)}{4}.
\]
In particular, when the smoothing covariance and the target covariance have the
same tail order, $\lambda_j\asymp\sigma_j\asymp j^{-a}$, the balanced
preconditioner is
\[
\gamma_j\asymp j^{-2a/3},
\]
and the sufficient conditions reduce to $a>3$.

\paragraph{Verification of the Sufficient Conditions.}
We now verify these claims directly from the conditions of
Theorem~\ref{thm:concrete-elp}.

\begin{proposition}
\label{prop:power-law}
Consider a Gaussian mixture with common diagonal covariance,
\[
\Sigma_i^d=\Sigma^d=\operatorname{Diag}(\sigma_1,\dots,\sigma_d),
\qquad i\in I,
\]
and assume that the component means are supported on finitely many coordinates:
there exists $J<\infty$ such that
\[
m_{ij}=0,
\qquad i\in I,\quad j>J .
\]
Assume that
\[
\sigma_j\asymp j^{-a},
\qquad
\lambda_j\asymp j^{-b},
\qquad
\gamma_j\asymp j^{-c},
\]
with $a>1$, $b\ge a$, and
\begin{equation}
\label{eq:power-law-range}
\max\left\{1,\frac{a+1}{2}\right\}
<
c
<
2b-a-1.
\end{equation}
Then the conditions
\eqref{eq:diag-step1-a}--\eqref{eq:diag-step3} hold. Consequently, the ELP
scheme satisfies the dimension-uniform KL bound of
Theorem~\ref{thm:concrete-elp}.
\end{proposition}

\begin{proof}
Since the covariance is common across mixture components, we have
\[
\underline\sigma_j=\overline\sigma_j=\sigma_j,
\qquad
\overline\sigma_j-\underline\sigma_j=0.
\]
Moreover, because the component means are supported on finitely many
coordinates, all terms involving $\overline m_j$ or
$\sup_{i,\ell}|m_{ij}-m_{\ell j}|$ vanish in the tail. Therefore these terms are harmless for summability.

We first check \eqref{eq:diag-step1-a}. Since $a>1$ and $b\ge a$, we have
\[
\sum_{j\ge1}\sigma_j<\infty,
\qquad
\sum_{j\ge1}\lambda_j<\infty.
\]
Together with the finite-support assumption on the means, this gives
\eqref{eq:diag-step1-a}.

Next consider \eqref{eq:diag-step1-b}. The condition
$\sum_j\gamma_j<\infty$ holds if $c>1$. For the second condition in
\eqref{eq:diag-step1-b}, since $b\ge a$, we have
$\sigma_j+\lambda_j\asymp\sigma_j$ in the tail. Hence
\[
\gamma_j^2
\frac{\sigma_j+\lambda_j}{\sigma_j^2}
\asymp
\frac{\gamma_j^2}{\sigma_j}
\asymp
j^{-2c+a}.
\]
This is summable if
\[
2c-a>1,
\qquad\text{that is,}\qquad
c>\frac{a+1}{2}.
\]
Thus the lower bound in \eqref{eq:power-law-range} gives
\eqref{eq:diag-step1-b}.

Condition \eqref{eq:diag-step1-c} is automatic in the tail because
$\overline\sigma_j-\underline\sigma_j=0$ and the means are supported on finitely many coordinates.
The same argument applies to the covariance-difference and mean-difference terms
in \eqref{eq:diag-step2-a}, \eqref{eq:diag-step2-bb}, and
\eqref{eq:diag-step2-d}.

It remains to check \eqref{eq:diag-step3}. In this setting it reads
\[
\sum_{j\ge1}
\frac{\lambda_j^2}{\gamma_j\sigma_j}
<\infty.
\]
Using the power laws,
\[
\frac{\lambda_j^2}{\gamma_j\sigma_j}
\asymp
j^{-2b+c+a}.
\]
Therefore this series is finite if
\[
2b-c-a>1,
\qquad\text{that is}\qquad
c<2b-a-1.
\]
Combining this upper bound with the lower bounds above gives exactly
\eqref{eq:power-law-range}. Hence all the assumptions
\eqref{eq:diag-step1-a}--\eqref{eq:diag-step3} hold, and the conclusion follows
from Theorem~\ref{thm:concrete-elp}.
\end{proof}

\paragraph{Interpretation of the Admissible Preconditioner.}
The interval \eqref{eq:power-law-range} has a simple interpretation:
\begin{itemize}
    \item The lower
bound on $c$ is the damping requirement. If $c$ is too small, then
$\gamma_j\asymp j^{-c}$ does not decay fast enough, and the high-frequency
discretization estimates diverge.
\item  The upper bound on $c$ is the annealing
requirement. If $c$ is too large, then $\gamma_j$ is too small in the tail, and
the annealing contribution
\[
\sum_{j\ge1}\frac{\lambda_j^2}{\gamma_j\sigma_j}
\]
diverges. 
\end{itemize}
In other words, the preconditioner must
be damped enough for discretization, but not so damped that annealing control is
lost.

\paragraph{Balanced Preconditioner.}
The same trade-off can be seen directly, without assuming a power law for
$\gamma_j$. In this simplified setting, dimension-uniform control follows
from controlling the two tail contributions with per-coordinate terms
\[
\frac{\gamma_j^2}{\sigma_j}
\qquad\text{and}\qquad
\frac{\lambda_j^2}{\gamma_j\sigma_j}.
\]
The first term comes from the drift and discretization estimates; it penalizes
large high-frequency preconditioning. The second term comes from the annealing
contribution; it penalizes excessive damping of the preconditioner. Balancing
the two terms coordinatewise gives
\[
\frac{\gamma_j^2}{\sigma_j}
\asymp
\frac{\lambda_j^2}{\gamma_j\sigma_j},
\]
and hence
\[
\gamma_j\asymp\lambda_j^{2/3}.
\]
With this balanced choice,
\[
\frac{\gamma_j^2}{\sigma_j}
\asymp
\frac{\lambda_j^2}{\gamma_j\sigma_j}
\asymp
\frac{\lambda_j^{4/3}}{\sigma_j},
\]
and the sufficient conditions of
Theorem~\ref{thm:concrete-elp} reduce to
\[
\sum_{j\ge1}\frac{\lambda_j^{4/3}}{\sigma_j}<\infty.
\]
Under the power laws, this translates into the following proposition.
\begin{proposition}
\label{prop:power-law-balanced-preconditioner}
In the setting of Proposition~\ref{prop:power-law}, assume that
\[
\sigma_j\asymp j^{-a},
\qquad
\lambda_j\asymp j^{-b}, \qquad b\geq a.
\]
Then the preconditioner that balances the need to accelerate the annealing dynamics
with the need to damp high-frequency discretization errors is
\[
\gamma_j\asymp\lambda_j^{2/3}\asymp j^{-2b/3}.
\]
With this balanced choice,
\[
\frac{\gamma_j^2}{\sigma_j}
\asymp
\frac{\lambda_j^2}{\gamma_j\sigma_j}
\asymp
\frac{\lambda_j^{4/3}}{\sigma_j}
\asymp
j^{-4b/3+a}.
\]
Therefore the sufficient conditions of
Theorem~\ref{thm:concrete-elp} are satisfied if
\[
\sum_{j\ge1}j^{-4b/3+a}<\infty,
\]
that is
\[
b>\frac{3(a+1)}{4}.
\]
In particular, if
\[
\lambda_j\asymp\sigma_j\asymp j^{-a},
\]
the balanced preconditioner is
\[
\gamma_j\asymp j^{-2a/3},
\]
and the sufficient conditions of 
Theorem~\ref{thm:concrete-elp} reduce to $a>3$.
\end{proposition}

\section{Additional Details on the Numerical Experiments}
\label{app:numerics-details}

This appendix collects implementation details and additional diagnostics for the
experiments in Section~\ref{sec:numerics}. All figures were generated in Google
Colab with 13 GB of RAM, with a runtime of about one minute.

\paragraph{Experimental Setup.}
In Figure~\ref{fig:em-elp-combined}, we compare Euler--Maruyama (EM) and the
exact-linear-part (ELP) discretization on the same preconditioned ALD dynamics.
The target is the two-component Gaussian mixture
\[
\rho_\star^d
=
w_1\mathcal N(m_1^d,\Sigma^d)
+
w_2\mathcal N(m_2^d,\Sigma^d),
\]
with
\[
w_1=0.75,\qquad w_2=0.25,\qquad m_1^d=\mathbf{0},\qquad m_2^d=8e_1,
\]
and common covariance
\[
\Sigma^d=\operatorname{Diag}(\sigma_j)_{j=1}^d,
\qquad
\sigma_j=j^{-6}.
\]
The smoothing covariance and preconditioner are diagonal, with
\[
\lambda_j=j^{-6},
\qquad
\gamma_j=j^{-4}.
\]
The annealing path is
\[
\rho_t^d
=
\rho_\star^d*\mathcal N(0,\theta(t)C^d),
\qquad
\theta(t)=2S(1-t/T),
\qquad
S=5.
\]
Both schemes are initialized exactly from
\[
\rho_0^d=\rho_\star^d*\mathcal N(0,2SC^d).
\]
We use step size $h=10^{-3}$, $2500$ time steps, and hence $T=2.5$. The
dimensions are
\[
d\in\{1,5,10,20,30,40,50,60\}.
\]

\paragraph{Spectral Summability Conditions.}
Since
$\lambda_j=\sigma_j$, the initialization becomes increasingly far from the
target:
\[
\KL(\rho_\star^d\|\rho_0^d)
\]
grows linearly with $d$. At the same time, the ELP summability conditions hold;
in particular,
\[
\sum_{j\ge1}\sigma_j<\infty,
\qquad
\sum_{j\ge1}\lambda_j<\infty,
\qquad
\sum_{j\ge1}\gamma_j<\infty,
\qquad
\sum_{j\ge1}\frac{\lambda_j^2}{\gamma_j\sigma_j}<\infty.
\]
The two mixture components have the same covariance, so the covariance-mismatch
terms in the ELP assumptions vanish.

For EM, however,
\[
\frac{\gamma_j}{\sigma_j}=j^2.
\]
Thus the largest linearly stable EM step size up to dimension $d$ is of order
\[
h_{\max}^{\mathrm{EM}}(d)\simeq \frac{2}{d^2}.
\]
Therefore, for a fixed time step, the EM stability condition is violated once the
dimension is large enough; in the experiment this happens around $d\approx45$.

\paragraph{KL Divergence Estimation.}
The left panel of Figure~\ref{fig:em-elp-combined} reports a fixed-$k$ nearest-neighbor estimate of
$\KL(\rho_\star^d\|\Law(Y_T^d))$, following the KL estimators of
\cite{perez2008kullback, wang2009divergence}. We use $k=20$ in the main text. For each
dimension, we generate $800$ target samples and $800$ terminal samples from the discretized sampler.

As a robustness check, we repeat the estimation with
$k\in\{10,20,30,50\}$. Figure~\ref{fig:knn-k-robustness} shows the same
behavior across these choices: EM estimate grows rapidly in the stiff
regime, whereas the ELP estimate remains stable.

\begin{figure}[H]
    \centering
    \includegraphics[width=0.65\linewidth]{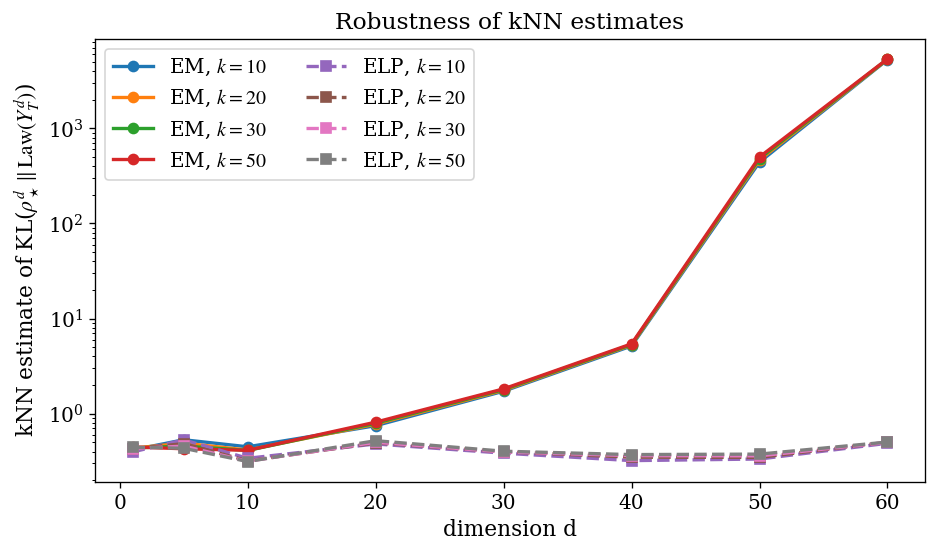}
    \caption{$k$NN KL estimates for different choices of $k$. The qualitative behavior is
stable across $k\in\{10,20,30,50\}$ (the main text reports $k=20$): EM grows rapidly once it enters the high-frequency
stiffness regime, whereas ELP remains stable.}
    \label{fig:knn-k-robustness}
\end{figure}

\paragraph{Variance Profile.}
The right panel of Figure~\ref{fig:em-elp-combined} shows the coordinate-wise variance profile at $d=50$. For each coordinate $j$, we compute the empirical variance of the terminal samples and normalize it by the corresponding target marginal variance. This identifies the coordinates in which the instability occurs: EM develops a large variance excess in the high-frequency coordinates, where $\gamma_j/\sigma_j$ is largest, whereas ELP remains close to the target scale.

%%%%%%%%%%%%%%%%%%%%%%%%%%%%%%%%%%%%%%%%%%%%%%%%%%%%%%%%%%%%

\end{document}